\documentclass{article}

\usepackage[preprint]{neurips_2023}

\usepackage[utf8]{inputenc} %
\usepackage[T1]{fontenc}    %
\usepackage{hyperref}       %
\usepackage{url}            %
\usepackage{booktabs}       %
\usepackage{amsfonts}       %
\usepackage{nicefrac}       %
\usepackage{microtype}      %
\usepackage{xcolor}         %

\usepackage{listings}
\usepackage{float}
\usepackage{fancyhdr}
\usepackage[pdftex]{graphicx}
\usepackage{amsmath}
\usepackage{amsthm}
\usepackage{amssymb}
\usepackage{listing}
\usepackage{enumitem}
\usepackage{apptools}
\usepackage{chngcntr}

\usepackage{mathtools}
\usepackage{wrapfig}
\usepackage{color}
\usepackage[ruled, vlined, linesnumbered]{algorithm2e}
\DeclareMathOperator*{\argmin}{argmin}
\DeclareMathOperator*{\argmax}{argmax}

\usepackage[capitalise,nameinlink,sort]{cleveref}
\Crefname{equation}{Eq.}{Eqs.}
\Crefname{assumption}{Assumption}{Assumptions}
\Crefname{condition}{Condition}{Conditions}
\Crefname{problem}{Linear Program}{Linear Programs}
\Crefname{question}{Questions}{Questions}

\usepackage{quoting}

\usepackage[parfill]{parskip}

\usepackage{tikz}

\tikzstyle{round}=[thick,draw=black,circle]

\usepackage{thm-restate}

\theoremstyle{plain}
\newtheorem{theorem}{\textbf{Theorem}}
\newtheorem{lemma}{\textbf{Lemma}}

\newtheorem{proposition}[lemma]{Proposition}

\newtheorem*{lemma-a}{\textbf{Lemma}}
\newtheorem*{theorem-a}{\textbf{Theorem}}

\theoremstyle{definition}
\newtheorem{assumption}{Assumption}

\newtheorem{remark}{Remark}
\hypersetup{colorlinks=true,citecolor=blue,linkcolor=red}
\setcitestyle{citesep={;}}

\title{Offline Reinforcement Learning with Additional Covering Distributions}

\author{%
  Chenjie Mao\\
  School of Computer Science and Technology\\
  Huazhong University of Science and Technology\\
  Wuhan 430074, China\\
  \texttt{chenjiemao@hust.edu.cn} \\
}

\begin{document}

\maketitle

\begin{abstract}
    We study learning optimal policies from a logged dataset, i.e., offline RL, with function approximation. Despite the efforts devoted, existing algorithms with theoretic finite-sample guarantees typically assume exploratory data coverage or strong realizable function classes, which is hard to be satisfied in reality. While there are recent works that successfully tackle these strong assumptions, they either require the gap assumptions that only could be satisfied by part of MDPs or use the behavior regularization that makes the optimality of learned policy even intractable. To solve this challenge, we provide finite-sample guarantees for a simple algorithm based on marginalized importance sampling (MIS), showing that sample-efficient offline RL for general MDPs is possible with only a partial coverage dataset and weak realizable function classes given additional side information of a covering distribution. Furthermore, we demonstrate that the covering distribution trades off prior knowledge of the optimal trajectories against the coverage requirement of the dataset, revealing the effect of this inductive bias in the learning processes.
    \end{abstract}

\section{Introduction and related works}
In offline reinforcement learning (offline RL, also known as batch RL), the learner tries to find good policies with a pre-collected dataset.
This data-driven paradigm eliminates the heavy burden of environmental interaction required in online learning, which could be dangerous or costly~%
(e.g., in robotics~\citep{Kalashnikov2018QTOptSD,Sinha2021S4RLSS} and healthcare~\citep{Gottesman2018EvaluatingRL,Gottesman2019GuidelinesFR,Tang2022LeveragingFA}), 
making offline RL a promising approach in real-world applications.

In early theoretic studies of offline RL (e.g., \citet{Munos2003ErrorBF,Munos2005ErrorBF,Munos2007PerformanceBI,Ernst2005TreeBasedBM,Antos2007FittedQI,Munos2008FiniteTimeBF,Farahmand2010ErrorPF}), 
researchers analyzed the finite-sample behavior of algorithms under the assumptions such as \textit{exploratory} datasets,
realizable or Bellman-complete function classes.
However, despite some error propagation bounds and sample complexity guarantees achieved in these works, %
the strong coverage assumption made on datasets and the \textit{non-monotonic} assumptions made on function classes---%
which are always hard to be satisfied in reality---%
drive people to try to find sample-efficient offline RL algorithms under only weak assumptions about dataset and function classes~\citep{Chen2019InformationTheoreticCI}.

From the \textit{dataset perspective}, 
partial coverage, which means that only some specific (or even none) policies are covered by the dataset~\citep{Rashidinejad2021BridgingOR,Xie2021BellmanconsistentPF,Uehara2021PessimisticMO,Song2022HybridRU}, is studied. 
To address the problem of insufficient information, 
most algorithms use \textit{behavior regularization} (e.g.,
\citet{Laroche2017SafePI,Kumar2019StabilizingOQ,Zhan2022OfflineRL})
or \textit{pessimism in the face of uncertainty}
(e.g., \citet{Liu2020ProvablyGB,Jin2020IsPP,Rashidinejad2021BridgingOR,Xie2021BellmanconsistentPF,Uehara2021PessimisticMO,Cheng2022AdversariallyTA,Zhu2023ImportanceWA}) 
to constrain the learned policies to be close to the behavior policy.
Most of the algorithms in this setting (except some that we will discuss later) require function assumptions in some sense of \textit{completeness}%
---Bellman-completeness or strict realization 
according to another function class (we attribute it as strong realization).

From the \textit{function classes perspective}, while the primary concern is Bellman-completeness assumption
which is criticized for its non-monotonicity, 
some recent works~\citep{Zhan2022OfflineRL,Chen2022OfflineRL,Ozdaglar2022RevisitingTL} have noticed that the realizability according to another function class is also non-monotonic. %
These non-monotonic properties contradict the intuition in supervised learning that rich function classes perform better (or at least no worse).
Typical examples of these assumptions are the ``realizability of all candidate policies' value functions'' (e.g., \citet{Jiang2020MinimaxVI,Zhu2023ImportanceWA}) and the ``realizability of all candidate policies' density ratio'' (e.g., \citet{Xie2020QAS}).
These assumptions are equally strong as Bellman-completeness, and we classify them as ``strong realizability'' (\citet{Zhan2022OfflineRL,Ozdaglar2022RevisitingTL} attribute it as ``completeness-type'') for clarification.
We also classify assuming that the function class realizes specific elements as ``weak realizability'' correspondingly (\citet{Chen2022OfflineRL} attributes this as ``realizability-type''). 
We argue that this taxonomy is justified also because Bellman-completeness can be converted to the realizability assumption between two function classes with the minimax algorithm~\citep{Chen2019InformationTheoreticCI}.
This conversion aligns the behavior of Bellman-completeness with strong realizability assumptions.

On the one hand, Bellman-completeness assumption is always made in the classical finite-sample analyses of offline RL (e.g., analysis of 
FQI~\citep{Ernst2005TreeBasedBM,Antos2007FittedQI}) to ensure closed updates of value functions~\citep{Sutton1998,Wang2021InstabilitiesOO}.
This assumption is notoriously hard to mitigate, and \citet{Foster2021OfflineRL} even suggests an information-theoretic lower bound stating that without Bellman-completeness, 
sample-efficient offline RL is impossible even with an exploratory dataset and a function class containing all candidate policies' value functions.
Therefore, it is clear that additional assumptions are required to circumvent Bellman-completeness.

On the other hand, as marginalized importance sampling (MIS, see, e.g., \citet{Liu2018BreakingTC,Uehara2019MinimaxWA,Jiang2020MinimaxVI,Huang2022BeyondTR}) has shown its effect of 
eliminating Bellman-completeness with only a partial coverage dataset by assuming the realizability of density ratios in off-policy evaluation (OPE),
there are works try to adapt it to policy optimization. 
These adaptations retain the elimination of Bellman-completeness, but most come up with other drawbacks.%
Some works (e.g., \citet{Jiang2020MinimaxVI,Zhu2023ImportanceWA}) use OPE as an intermediate evaluation step for policy optimization yet require the strong realizability assumption on the value function class.
The others borrow the idea of discriminators from MIS. \citet{Lee2021OptiDICEOP,Zhan2022OfflineRL} take value functions as discriminators for the optimal density ratio, using MIS 
to approximate the linear programming approach of Markov Decision Processes~\citep{Manne1960M,Puterman1994MarkovDP}.
\citet{Nachum2019AlgaeDICEPG,Chen2022OfflineRL,Uehara2023RefinedVO} take distribution density ratios as discriminators for optimal value function by replacing the Bellman equation in OPE with its optimality variant.
While in most cases, 
theoretic finite-sample guarantees with these discriminators would require strong realizable function classes (e.g., \citet{Rashidinejad2022OptimalCO}), 
\citet{Zhan2022OfflineRL,Chen2022OfflineRL,Uehara2023RefinedVO} avoid this with additional gap assumptions or an alternative criterion of optimality---performance degradation w.r.t. the regularized optimal policy. 
To the best of our knowledge, they are the only works that achieve theoretic sample-efficient guarantees under only weak realizability and partial coverage assumptions.
However, on the one hand, the gap (margin) assumption~\citep{Chen2022OfflineRL,Uehara2023RefinedVO} causes that only some specific Markov decision processes (MDPs)---under which the optimal value functions have gaps---can be solved.
On the other hand, sub-optimality compared with a regularized optimal policy~\citep{Zhan2022OfflineRL} could be meaningless in some cases, and the actual performance of the learned policy is intractable.

As summarized above, the following question arises:
\begin{center}
    \textit{Is sample-efficient offline RL possible with only partial coverage and weak realizability assumptions for general MDPs?}
\end{center}
We answer this question in the positive and propose an algorithm that achieves finite-sample guarantees under weak assumptions with the help of an additional covering distribution.
We assume that the covering distribution covers all non-stationary near-optimal policies, and the dataset covers the trajectories induced by an optimal policy from it. %
The covering distribution is \textit{adaptive} such that both ``non-stationary'' and ``near-optimal'' above 
would be alleviated as the gap of optimal value function increases.
The covering distribution also gives a trade-off against the data coverage assumption: 
the more accurate it is, the fewer redundant trajectories are required to be covered by the dataset. 
Furthermore, 
we can directly use the data distribution as the covering distribution as done in \citet{Uehara2023RefinedVO}, 
if the near-optimal variant of their data assumptions are also satisfied.

For comparison, we summarize algorithms with partial coverage that do not need 
Bellman-completeness and
\textit{model realizability} (which is even stronger~\citep{Chen2019InformationTheoreticCI,Zhan2022OfflineRL}) 
in \Cref{table:assumptions}.
Necessary transfers are made to get the sub-optimality bound.
We have removed additional definitions of notations for simplicity and refer the interested reader to the original papers for more detail.

\begin{table}
  \tiny
  \centering
  \caption{Comparison of offline RL algorithm (conc. stand for concentrability)}
  \label{table:assumptions}
  \begin{tabular}{llll}
    \toprule
    Algorithm  & Data assumptions   & Function assumptions & Major drawbacks\\
    \midrule
    \citet{Jiang2020MinimaxVI}     & optimal conc.    & $w^\star \in\mathcal{W}$, and $\forall \pi \in\Pi, Q_{\pi}\in\mathcal{C}(\mathcal{Q})$&strong realizability\\
    \midrule
    \citet{Zhan2022OfflineRL}     & optimal conc.    & $w^\star_\alpha\in\mathcal{W}$, and $v^\star_\alpha\in \mathcal{V}$ & compare with $\pi^\star_{\alpha}$\\
    \midrule
    \citet{Chen2022OfflineRL}     & optimal conc.    & $w^\star\in\mathcal{W}$, and $Q^\star\in \mathcal{Q}$& assume gap (margin) \\
    \midrule
    \citet{Rashidinejad2022OptimalCO}    & optimal conc.    & $w^\star\in\mathcal{W}$, $V^\star\in\mathcal{V}$, $u^\star_w\in\mathcal{U}\ \forall w$ and $\zeta^\star_{w^\star,u}\in\mathcal{Z}\ \forall u$&strong realizability\\
    \midrule
    \citet{Zhu2023ImportanceWA}     & optimal conc.    & $w^\star\in\mathcal{W}$, and $\forall \pi\in\Pi, Q_{\pi}\in \mathcal{Q}$&strong realizability\\
    \midrule
    \citet{Uehara2023RefinedVO}    & optimal conc from $d^\mathcal{D}$    & $w^\star\in\mathcal{W}$, and $Q^\star\in \mathcal{Q}$ &assume gap (margin)\\
    \midrule
    Ours (VOPR)    & optimal conc. from $d_c$   & $w^\star\in\mathcal{W}$, $\beta^\star\in\mathcal{B}$ and $Q^\star\in \mathcal{Q}$ & assume a covering $d_c$\\
    \bottomrule
  \end{tabular}
\end{table}

In conclusion, our contributions are as follows: 
\begin{itemize}
    \item (\Cref{section:adv2gap}) We identify the novel mechanism of non-stationary near-optimal concentrability in policy optimization under weak assumptions. 
    \item (\Cref{section:algorithm-and-analysis}) We demonstrate the trade-off brought by additional covering distributions for the coverage requirement of the dataset.
    \item (\Cref{section:algorithm-and-analysis}) We propose the first algorithm that achieves finite-sample guarantees for general MDPs under only weak realizability and partial coverage assumptions.
\end{itemize}

\section{Preliminaries}
\label{preliminaries}

This section introduces base concepts and notations in offline RL with function approximation and MIS. See \Cref{notations-table} in \Cref{section:notations-table} for a more complete list of definitions of notations.

\paragraph{Markov Decision Processes (MDPs)}
We consider infinite-horizon discounted MDPs defined as $(\mathcal{S}, \mathcal{A}, P, R, \gamma, \mu_0)$, where 
$\mathcal{S}$ is the state space, $\mathcal{A}$ is the action space, 
$P\colon \mathcal{S}\times\mathcal{A}\to \Delta(\mathcal{S})$ is the transition probability, 
$R\colon \mathcal{S}\times \mathcal{A} \to [0, R_{max}]$ is the \textit{deterministic} reward function%
, $\gamma\in(0, 1)$ is the discount factor that unravels the problem of infinite horizons, 
and $\mu_0\in \Delta(\mathcal{S})$ is the initial state distribution.
With a policy $\pi\colon \mathcal{S}\to \Delta(\mathcal{A})$, 
we say that it induces a random trajectory $\{s_0, a_0, r_0, s_1, a_1, r_1, \dots, s_i, a_i, r_i, s_{i+1}, \dots\}$ if:
$s_0\sim \mu_0$, $a_i\sim \pi(\cdot|s_i)$, $r_i = R(s_i, a_i)$ and $s_{i+1}\sim P(\cdot|s_i, a_i)$.
We define the expected return of a policy $\pi$ as 
$J_\pi=\mathbb{E}\big[\sum_{i=0}^{\infty}\gamma^ir_i\mid \mu_0, \pi \big]$.
We also denote the value function of $\pi$ as the expected return starting from some specific state $s$ or state-action pair $(s, a)$ as
$V_\pi(s)=\mathbb{E}\big[\sum_{i=0}^{\infty}\gamma^ir_i\mid s, \pi \big]$ and
$Q_\pi(s, a)=\mathbb{E}\big[\sum_{i=0}^{\infty}\gamma^ir_i\mid (s, a), \pi \big]$.
We denote the optimal policies that achieve the maximum return $J^\star$ from $\mu_0$ as $\Pi^\star$, and its member as $\pi^\star$.
We say a policy is optimal almost everywhere if its state value function is maximized at every state and denote it as $\pi_e^\star$ ($\pi_e^\star$ is not always unique).
We represent the value functions of $\pi^\star_e$ as $V^\star$ and $Q^\star$.
It worth noting that $V^\star$ and $Q^\star$, the \textit{unique} solutions of both Bellman optimality equation and the primal part of LP approach of MDPs~\citep{Puterman1994MarkovDP}, are \textit{not} the value functions of all optimal policies. %
For ease of discussion, we assume $\mathcal{S}$, $\mathcal{A}$, $\mathcal{S}\times\mathcal{A}$ are compact measurable spaces and, 
with abuse of notation, we use $\nu$ to denote the corresponding finite uniform measure on each space (e.g., Lebesgue measure).
We use $P_\pi$ to denote the state-action transition operator for density $d$ as $P_\pi d(s^\prime, a^\prime)\coloneqq\int_{\mathcal{S}\times\mathcal{A}}\pi(a^\prime\mid s^\prime)P(s^\prime\mid s, a)d(s, a)d\nu(s, a)$.
\paragraph{Offline policy learning with function approximation}
In the standard theoretical setup of offline RL, 
we are given with a dataset $\mathcal{D}$ consisting of $N$ $(s, a, r, s^\prime)$ tuples, 
which is collected with \textit{state} distribution $\mu^D$ and behavior policy 
$\pi_b$
such that $s\sim\mu^D, a\sim\pi_b(\cdot|s), r=R(s, a), s^\prime\sim P(\cdot|s, a)$.
We use $d^\mathcal{D}(s, a)\coloneqq\mu(s)\pi_b(a\mid s)$
to denote the composed state-action distribution of the dataset.
When the state space and action space become rather complex, function approximation is typically used.
For this, we assume there are some function classes at hand that satisfy certain assumptions and have limited complexity (measured by cardinality, metric entropy and so forth).
The function classes considered in this paper
are state-action value function class $\mathcal{Q}\subseteq(\mathcal{S}\times\mathcal{A}\to\mathbb{R})$, state distribution ratio class $\mathcal{W}\subseteq(\mathcal{S}\times\mathcal{A}\to\mathbb{R})$, and policy ratio class $\mathcal{B}\subseteq(\mathcal{S}\times\mathcal{A}\to\mathbb{R})$.

\paragraph{MIS with density discriminators and $L^2$ error bound} %
One of the most popular ways to estimate the optimal value function is via the Bellman optimality equation:
\begin{align}
    \forall s\in\mathcal{S}, a\in\mathcal{A}, \quad Q^\star(s, a) = T^\star Q^\star(s, a)\label{eqn:Bellman-opt}
\end{align}
where $T^\star q(s, a) \coloneqq R(s, a) + \gamma \mathbb{E}_{s^\prime\sim P(\cdot\mid s, a)}[\max q(s^\prime, \cdot)]$ denotes the Bellman optimality operator.
However, when we try to utilize the constraints from \Cref{eqn:Bellman-opt} (e.g., through the $L^1$ error $\lVert q - T^\star q\rVert_{1, d^\mathcal{D}}$),
the expectation in $T^\star$ would introduce the infamous double-sampling issue~\citep{Baird1995ResidualAR}, 
making the estimation intractable. %
To overcome this, privious works with MIS tried to take weight functions as discriminators and
minimize a weighted sum of \Cref{eqn:Bellman-opt}.
In fact, 
even the $L^1$ error itself could be written as a weighted sum
with some sign function (take $1$ if $q \ge T^\star q$ and $-1$ otherwise~\citep{Ozdaglar2022RevisitingTL}). 
Namely, we approximate $Q^\star$ through
\begin{align}
    \hat{q} = \argmin\limits_{q\in\mathcal{Q}}
    \max\limits_{w\in\mathcal{W}}
    \mathbb{E}_{d^\mathcal{D}} [w(s, a)(q(s,a) - T^\star q(s, a)]. \label{eqn:w-bellman-error}
\end{align}
Since the weight function class $\mathcal{W}$ is marginalized into the state-action space (instead of trajectories), this approach is called marginalized importance sampling (MIS)~\citep{Liu2018BreakingTC}. 
While theoretic guarantees in MIS under weak realizability and partial coverage assumptions are typically made for scalar values (e.g., the return~\citep{Uehara2019MinimaxWA,Jiang2020MinimaxVI}), 
recently, \citet{Zhan2022OfflineRL,Huang2022BeyondTR,Uehara2023RefinedVO} have gone beyond this and derived $L^2$ error guarantees for the estimators by using some strongly convex functions. 
Among them, the optimal value function estimator from \citet{Uehara2023RefinedVO} constructs the base of this work.

\section{From $Q^\star$ to optimal policy, the minimum requirement}
\label{section:adv2gap}
\citet{Uehara2023RefinedVO} shows that accurately estimating optimal value function $Q^\star$ under 
$d^\mathcal{D}$ is possible if $d^\mathcal{D}$ covers the optimal trajectories starting from itself. 
This ``self-covering'' assumption could be relieved and generalized if we only require an accurate estimator under some state-action distribution $d_c$ such that $d_c\ll d^\mathcal{D}$ (we use $\mu_c$ and $\pi_c$ to denote the state distribution and policy decomposed from $d_c$).
In fact, $d_c$ provides a trade-off for the coverage requirement of the dataset: the fewer state-action pairs on the support of $d_c$, the weaker data coverage assumptions we will make.
Nevertheless, how much trade-off can $d_c$ provide while preserving the desired result?

In policy learning, 
our goal is to derive an optimal policy $\hat{\pi}$ from the estimated $Q^\star$ (denoted as $\hat{q}$).
While there are methods (see \Cref{subsection:comparison} for a brief discussion) that induce policies from $\hat{q}$ by exploiting pessimism or data regularization,
one of the most straightforward ways is to take the actions covered by $d_c$ that achieve the maximum $\hat{q}$ in each state.
This can be done with the help of policy ratio class $\mathcal{B}$, via
\begin{align}
    \hat{\beta} = \argmax\limits_{\beta\in\mathcal{B}}\langle \mu_c, \hat{q}(\cdot, \pi_{\beta})\rangle\quad\textup{and take}\quad \hat{\pi} = \pi_{\hat{\beta}}, \label{eqn:est1}
\end{align}
where $\pi_{\beta}(a\mid s) = \pi_b(a\mid s)\beta(s, a)$ (normalized if needed).
With the optimal realizability of $\mathcal{B}$ and concentrability of $\pi_c$, \Cref{eqn:est1} is actually equivalent to 
\begin{align}
    \langle \mu_c, Q^\star(\cdot, \hat{\pi}) - Q^\star(\cdot, \pi^\star_e)\rangle=0, \label{eqn:0-adv}
\end{align}
which guides us to exploit the coverage provided by $\mu_c$.
Recall that our goal is to use $d_c$ to trade off the coverage assumption of $d^\mathcal{D}$. Therefore, the question left, which forms the primary subject of this section, is
\begin{center}
    \textit{With which $\mu_c$ can we conclude that $\hat{\pi}$ is optimal from $\langle \mu_c, Q^\star(\cdot, \hat{\pi}) - Q^\star(\cdot, \pi^\star_e)\rangle=0$,\\and what is the minimum requirement of it?}
\end{center}
Since $\mu_c$ and $d_c$ are to provide additional coverage, we also call them ``covering distributions''.

The remainder of this section is organized as follows: we first show why 
single optimal concentrability of $\mu_c$ is not enough in \Cref{subsection:dilemma}, 
and then we introduce the alternative ``all optimal concentrability'' in \Cref{subsection:induction} and 
the adapted version of it in \Cref{subsection:stat} to  accommodate statistical errors.

\subsection{The dilemma of single optimal contentrability}
\label{subsection:dilemma}
Single optimal concentrability is standard~\citep{Liu2020ProvablyGB,Xie2021BellmanconsistentPF,Cheng2022AdversariallyTA} when we try to mitigate exploratory data assumptions (e.g., all-policy concentrability).
However, this framework suffers from a conundrum if only making weak realizability assumptions: we will know that the learned policy performs well only if we are informed with trajectories induced by it---%
rather than the ones induced by the covered policy.

More specifically, 
as the optimality of $\hat{\pi}$ could be quantified as $J^\star - J_{\hat{\pi}}$, the performance gap, 
we can telescope it through the performance difference lemma.
\begin{lemma}[The performance difference lemma] We can decompose the performance gap as 
\label{lemma-a:pdl}
    \begin{align*}
        (1-\gamma)(J_{\pi_1} - J_{\pi_2}) = %
        \langle \mu_{\pi_1}, Q_{\pi_2}(\cdot, \pi_1) - Q_{\pi_2}(\cdot, \pi_2)\rangle.
    \end{align*}
\end{lemma}
Thus, with \Cref{eqn:0-adv}, if we want $J^\star-J_{\hat{\pi}}$ (i.e., $J_{\pi^\star_e}-J_{\hat{\pi}}$) to be equal to zero, 
it might be necessary to require $\mu_c$ to cover $\mu_{\hat{\pi}}$ ($\mu_c\gg \mu_{\hat{\pi}}$) since the right part of the inner product in \Cref{eqn:0-adv} is always non-positive.
However, as $\hat{\pi}$ is estimated and is even random when considering approximating it from a dataset, 
$\mu_c\gg\mu_{\hat{\pi}}$ is usually achieved through \textit{all-policy concentrability}---$\mu_c\gg \mu_{\pi}$ for all $\pi$ in the hypothesis class.
Single optimal concentrability fails to provide the desired result. 

For instance, consider the counterexample in \Cref{fig:lower_bound} which is adapted from \citet{Zhan2022OfflineRL,Chen2022OfflineRL}. 
Suppose we finally get the following covering distribution and policy:
\begin{align*}
\mu_c(s) = \begin{cases}
    \nicefrac{1}{2} & \textup{if $s=1$}\\
    \nicefrac{1}{2} & \textup{if $s=2$}
\end{cases}
\quad \textup{and}\quad 
\hat{\pi}(s) = \begin{cases}
    \textup{L} & \textup{if $s=1$}\\
    \textup{R} & \textup{if $s=3$}\\
    \textup{Random} & \textup{elsewise}.
\end{cases}
\end{align*}
While $\mu_c$ achieves single optimal concentrability and $\hat{\pi}$ achieves the maximized value of $Q^\star$ in each state on the support of $\mu_c$, $\hat{\pi}$ is not an optimal policy since
it would end up with $0$ return.

\begin{figure}
\includegraphics[width=0.23\linewidth]{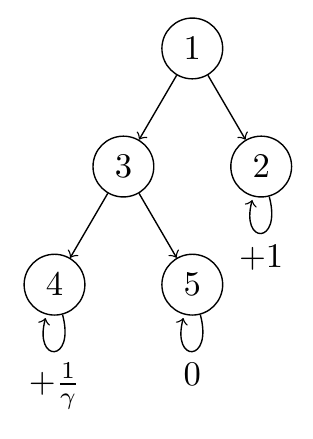}
\centering
\caption{%
The above MDP is deterministic, and we initially start from state $1$. We can take actions $L$ (left) and $R$ (right) in each state. In
states $1$ and $3$, action $L$ ($R$) will transfer us to its left (right) hand state, and taking actions in other states will only cause a self-loop.
We can only obtain non-zero rewards by taking actions in states $2$ and $4$, with values $1$ and $\frac{1}{\gamma}$ correspondingly.
There are two trajectories that could lead to the optimal $\nicefrac{\gamma}{(1-\gamma)}$ return: 
$\{(1, \textup{R}), 2, \dots\}$ and $\{(1, \textup{L}), (3, \textup{L}), 4\dots\}$.
We take $\gamma$ as the discount factor.
}
\centering
\label{fig:lower_bound}
\end{figure}

\paragraph{How gap assumptions avoid this}
While both \citet{Chen2022OfflineRL} and \citet{Uehara2023RefinedVO} consider single optimal concentrability and weak realizability assumptions (\citet{Uehara2023RefinedVO} also assumes additional structures of the dataset), 
the gap (margin) assumptions ensure that only taking $\pi^\star$ as $\hat{\pi}$ could achieve \Cref{eqn:0-adv}.
Moreover, \citet{Chen2022OfflineRL} shows that with the gap assumption, 
we can even use a value-based algorithm to derive a near-optimal policy without accurately estimating $Q^\star$.

\subsection{All-optimal concentrability}
\label{subsection:induction}
While single optimal concentrability suffers the hardness revealed before, 
there is still an alternative for the exploratory covering $\mu_c$, which is shown in the following lemma:
\begin{lemma}[From advantage to optimality] If $\mu_c$ covers all distributions induced by \textup{non-stationary} optimal policies (i.e., $\mu_c\gg \mu_{\pi^\star_{\textup{non}}}$ for any $\pi^\star_{\textup{non}}$)
and \Cref{eqn:0-adv} holds, then $\hat{\pi}$ is optimal and $J_{\hat{\pi}} = J^\star$.
\label{lemma:popu-to-optimal}
\end{lemma}
\begin{remark}
\label{remark:non-s}
Non-stationary policies are frequently employed in the analysis of offline RL~\citep{Munos2003ErrorBF,Munos2005ErrorBF,Scherrer2012OnTU,Chen2019InformationTheoreticCI,Liu2020ProvablyGB}.
If we make the gap assumption, the ``all non-stationary'' requirement is discardable since the action in each state
that could lead to the optimal return is unique.
\end{remark}
\begin{remark}
    \cite{Wang2022OnGB}
    also utilizes the all-optimal concentrability assumption, but they consider the tabular setting 
    and they require additionally gap assumptions to achieve the near-optimal guarantees.
\end{remark}
We now provide a short proof of \Cref{lemma:popu-to-optimal}, showing by induction that $\hat{\pi}_i$---the non-stationary policy that adopts $\hat{\pi}$ at the beginning $0$-th to $i$-th (include the $i$-th) steps and then follows $\pi^\star_e$---is optimal. 
\begin{proof}
We first rewrite the telescoping equation in the performance difference lemma as
\begin{align}
    (1-\gamma)(J_{\hat{\pi}_i} - J^\star) =&
    \langle \mu_{\hat{\pi}_i}, Q^\star(\cdot, \hat{\pi}_i) - Q^\star(\cdot, \pi^\star_e)\rangle\\=&
    \langle \mu_{\hat{\pi}_i}^{0:i}, Q^\star(\cdot, \hat{\pi}) - Q^\star(\cdot, \pi^\star_e)\rangle + \langle \mu_{\hat{\pi}_i}^{i+1:\infty}, Q^\star(\cdot, \pi^\star_e) - Q^\star(\cdot, \pi^\star_e)\rangle\\=&
    \langle \mu_{\hat{\pi}_i}^{0:i}, Q^\star(\cdot, \hat{\pi}) - Q^\star(\cdot, \pi^\star_e)\rangle
\label{eqn:pdl-tele}
\end{align}
where $\mu^{i:j}_{\pi}$ denotes the $i$-th to $j$-th steps (include the $i$-th and $j$-th) part of $\mu_{\pi}$.
Thus, the optimality of $\hat{\pi}_i$ only depends on the first $0$-th to $i$-th steps, and $\hat{\pi}_i$ is optimal if this part is on the support of $\mu_c$.
Now we inductively show that, for any natural number $i$, the initial $0$-th to $i$-th steps part is covered: 
\begin{itemize}
    \item The step-$0$ part of $\mu_{\hat{\pi}}$ (i.e., $(1-\gamma)\mu_0$) is on the support of $\mu_c$ since there is some (non-stationary) optimal policy $\pi^\star$ covered by it, 
        \begin{align*}
            \mu_c\gg \mu_{\pi^\star}\gg \mu_0.
        \end{align*}
        Therefore, $\langle \mu_{\hat{\pi}_0}^{0:0}, Q^\star(\cdot, \hat{\pi}) - Q^\star(\cdot, \pi^\star_e)\rangle=0$.
        From \Cref{eqn:pdl-tele}, $\hat{\pi}_0$ is optimal.
    \item We next show that if $\hat{\pi}_i$ is optimal (which means that it's covered $\mu_c$), then the first $0$-th to $(i+1)$-th steps part of $\mu_{\hat{\pi}}$ is covered by $\mu_c$, which means that 
        $\hat{\pi}_{i+1}$ is optimal. This comes from the fact that the initial $0$-th to $(i+1)$-th steps part of the state distribution induced by a policy 
        only depends on its previous $0$-th to $i$-th decisions:
        \begin{align*}
            \mu_c\gg \mu_{\hat{\pi}_i}
            \gg \mu_{\hat{\pi}_i}^{0:i+1} = \mu_{\hat{\pi}}^{0:i+1}.
        \end{align*}
        From \Cref{eqn:pdl-tele}, $\hat{\pi}_{i+1}$ is optimal.
\end{itemize}
    Thus, for any $\epsilon>0$, there exists natural number $i\ge\log_{\gamma} \frac{\epsilon}{V_{\max}}$ such that 
    \begin{align*}
        J^\star-J_{\hat{\pi}}\le J^\star- J_{\hat{\pi}}^{0:i}\le J^\star - (J_{\hat{\pi}_i}-\gamma^{i+1} V_{\textup{max}})\le 
        \gamma^{i+1}V_{\textup{max}}\le \epsilon, 
    \end{align*}
    where $J_{\pi}^{i:j}$ denotes the $i$-th to $j$-th steps part of the return.
    Therefore, $\hat{\pi}$ is optimal.
\end{proof}
Consequently, instead of the exploratory data assumption, all non-stationary optimal coverage is sufficient for policy optimization.

\subsection{Dealing with statistical error}
\label{subsection:stat}

While \Cref{lemma:popu-to-optimal} is adequate at the population level (i.e., with an infinite amount of data), 
covering all non-stationary optimal policies is not enough 
when considering the empirical setting (i.e., with finite samples)
due to the introduced statistical error. 
This motivates us to adapt \Cref{lemma:popu-to-optimal} with a more refined $\mu_c$.%
\begin{assumption}[All near-optimal concentrability] 
\label{assumption:d-c}
We are given with a covering distribution $d_c$ such that its state distribution part $\mu_c$ covers the distributions induced by any non-stationary $\varepsilon_c$ near-optimal policy $\tilde{\pi}$:
\begin{align}
    \Big\lVert \frac{\mu_{\tilde{\pi}}}{\mu_c}\Big\rVert_{\infty}\le C_c,\quad\forall\tilde{\pi}\in\Pi_{\varepsilon_c,\textup{non}}^\star.
    \label{initial-shift}%
\end{align}
\end{assumption} %
We call a policy $\pi$ is $\varepsilon$ near-optimal if $J^\star - J_{\pi}\le\varepsilon$ 
and denote $\Pi_{\varepsilon,\textup{non}}^\star$ as the class of all non-stationary $\varepsilon$ near-optimal policies.
We also define $\frac{0}{0} = 1$ to suppress the extreme cases.
With this refined $\mu_c$, we can now derive the optimality of $\hat{\pi}$ even with some statistical errors.
\begin{lemma}[From advantage to optimality, with statistical errors]
\label{lemma:adv_to_subopt}
If $\langle \mu_c, Q^\star(\cdot, \hat{\pi}) - Q^\star(\cdot, \pi^\star)\rangle\ge-\varepsilon$ , 
and \Cref{assumption:d-c} holds with $\varepsilon_c\ge\frac{C_c\varepsilon}{1-\gamma}$,   
$\hat{\pi}$ is $\frac{C_c\varepsilon}{1-\gamma}$ near-optimal.
\end{lemma}
We defer the proof of this lemma to \Cref{proof-adv_to_subopt}.
\begin{remark}[The asymptotic property of $\varepsilon_c$]
\label{remark:asymptotic}
One of the most important properties of all near-optimal concentrability is that $\varepsilon_c$ depends on the statistical error,
which decreases as the amount of data increases.
\end{remark}

\section{Algorithm and analysis}
\label{section:algorithm-and-analysis}

After discussing the minimum requirement of estimating $Q^\star$, 
this section will demonstrate how to fulfill it and accomplish the policy learning task. 
Our algorithm, which is based on the optimal value estimator from \citet{Uehara2023RefinedVO}, 
follows the pseudocode in \Cref{algo}.

\IncMargin{1em}
\begin{algorithm}
\caption{VOPR (Value-Based Offline RL with Policy Ratio)}\label{algo}
\SetKwData{Left}{left}\SetKwData{This}{this}\SetKwData{Up}{up}
\SetKwFunction{Union}{Union}\SetKwFunction{FindCompress}{FindCompress}
\SetKwInOut{Input}{Input}\SetKwInOut{Output}{Output}
\Input{Dataset $\mathcal{D}$, value function class $\mathcal{Q}$, distribution density ratio class $\mathcal{W}$,
         policy ratio function class $\mathcal{B}$, and covering distribution $d_c$}
Estimate the optimal value function $\hat{q}$ as 
\begin{align}
    \hat{q} =  \argmin\limits_{q\in\mathcal{Q}}\max\limits_{w\in\mathcal{W}} \hat{\mathcal{L}}(d_c, q, w)
\end{align}
where
\begin{align}
    \hat{\mathcal{L}}(d, q, w) \coloneqq 0.5\mathbb{E}_{d}[q^2(s, a)]+ \frac{1}{N_{\mathcal{D}}}\sum\limits_{(s, a, r, s^\prime)\in\mathcal{D}} \Big[
    w(s, a)\big[\gamma \max q(s^\prime, \cdot)+r - q(s, a)\big]\Big]\label{eqn:emp-L}
\end{align}\\
Derive the approximated optimal policy ratio:
\begin{align*}
    \hat{\beta} = \argmax\limits_{\beta\in\mathcal{B}}\mathbb{E}_{\mu_c}
    [\hat{q}(s, \pi_{\beta})]
\end{align*}\\
\Output{$\hat{\pi} = \pi_{\hat{\beta}}$}
\end{algorithm}\DecMargin{1em}
We organized the rest of this section as follows: 
we first discuss the trade-off provided by the additional covering distribution $d_c$ and how to deduce $d_c$ in reality in \Cref{subsection:trade-off};
we then provide the finite-sample analysis of \Cref{algo} and
its proof sketch in \Cref{subsection:finite-sample};
we finally conclude this section by comparing our algorithms with the others in \Cref{subsection:comparison}.

We defer the main proofs in this section to \Cref{section:proof_algorithm}.

\subsection{Data assumptions and trade-off}
\label{subsection:trade-off}
As investigated in recent works~\citep{Huang2022BeyondTR,Uehara2023RefinedVO},
value function estimation under a given distribution requires a dataset that contains trajectories rolled out from it. 
Thus, our data assumption is as follows.
\begin{assumption}[Partial concentrability from $d_c$] %
\label{assumption:opt-pi-con}
The shift from $d^\mathcal{D}$ to the induced state-action distribution by $\pi_e^\star$ from $d_c$ is bounded:
\begin{align}
    \Big\lVert \frac{d_{d_c, \pi^\star_e}}{d^\mathcal{D}}\Big\rVert_{\infty}\le C_{\mathcal{D}}. \label{eqn:near-cover}%
\end{align}
\end{assumption} %
It is clear that with \Cref{assumption:opt-pi-con}, $d_c$ is also covered by $d^\mathcal{D}$.
\begin{proposition}
\label{pro:to-d}
If \Cref{assumption:opt-pi-con} holds, by definition of $d_{d_c, \pi^\star_e}$, 
    \begin{align*}
    \Big\lVert \frac{d_c}{d^\mathcal{D}}\Big\rVert_{\infty}\le
    \Big\lVert \frac{d_{d_c, \pi^\star_e}/(1-\gamma)}{d^\mathcal{D}}\Big\rVert_{\infty}\le \frac{C_{\mathcal{D}}}{1-\gamma}.
    \end{align*}
\end{proposition}
We now clarify the order of the learning process: we are first given with a dataset $\mathcal{D}$ with some good properties; 
then we try to find a $d_c$ from the support of the state-action distribution of $\mathcal{D}$ through some inductive bias (with necessary approximation);
finally, we apply \Cref{algo} with $\mathcal{D}$ and $d_c$.

The choice of $d_c$ constructs a trade-off between the knowledge about optimal policy and the requirement of data coverage. 
On the one hand, the most casual choice of $d_c$ is $d^\mathcal{D}$ (as in \citet{Uehara2023RefinedVO}),
which means we have no prior knowledge about optimal policies.
Employing $d^\mathcal{D}$ as $d_c$ will not only requires the dataset to cover unnecessary suboptimal trajectories, 
but also makes the dataset non-monotonic (adding new data points to it would break this assumption).
On the other hand, if we have perfect knowledge about optimal policies, \Cref{assumption:opt-pi-con} could be significantly alleviated.
More concretely, if $d_c$ strictly consists of the state-action distribution of trajectories induced by near-optimal policies,
our data assumption reduces to the per-step version of near-optimal concentrability.
\begin{lemma}
\label{lemma:exactly-d}
    If $d_c$ is a linear combination of the state-action distributions induced by non-stationary $\varepsilon$ near-optimal policies $\Pi_{\varepsilon, \textup{non}}^\star$ under a fixed probability measure $\lambda$:
    \begin{align}
        d_c = \int_{\Pi_{\varepsilon, \textup{non}}^\star} d_{\tilde{\pi}}d\lambda(\tilde{\pi}).\label{eqn:best-d}
    \end{align}
    And $d^\mathcal{D}$ covers all admissible distributions of $\Pi_{\varepsilon, \textup{non}}^\star$:
    \begin{align*}
        \forall\ \tilde{\pi}\in\Pi^\star_{\varepsilon, \textup{non}},\ i\in\mathbb{N}, \ 
        \Big\lVert\frac{d_{\tilde{\pi},i}}{d^\mathcal{D}}\Big\rVert_{\infty} \le C,
    \end{align*}
    where $d_{\pi,i}$ denotes the normalized distribution of the $i$-th step part of $d_{\pi}$.
    The distribution shift from $d^\mathcal{D}$ is bounded as 
    \begin{align*}
        \Big\lVert \frac{d_{d_c, \pi^\star_e}}{d^\mathcal{D}}\Big\rVert_{\infty}\le C.%
    \end{align*}
\end{lemma}
While the above case is impractical in reality,
it reveals the power of this inductive bias: the more auxiliary information we obtain about optimal paths, the weaker coverage assumptions of the dataset are required.

\subsection{Finite-sample guarantee}
\label{subsection:finite-sample}

We now give the finite-sample guarantee of \Cref{algo}, but before proceeding, we should state necessary function class assumptions.
The first are the weak realizability assumptions:
\begin{assumption}[Realizability of $\mathcal{W}$]
\label{assumption:real-w}
    There exists state-action distribution density ratio $w^\star\in \mathcal{W}$ such that $w^\star\circ d^\mathcal{D}=(I-\gamma P_{\pi^\star_e})^{-1}d_c Q^\star$.
\end{assumption}
\begin{assumption}[Realizability of $\mathcal{B}$]
\label{assumption:real-pi}
    There exists policy ratio $\beta^\star\in \mathcal{B}$ such that $\beta^\star\circ\pi_c=\pi^\star_e$  and for all $s\in\mathcal{S}, \int_{\mathcal{A}} \beta(s, a)\pi_c(s, a)d\nu(a)=1$.
\end{assumption}
\begin{assumption}[Realizability of $\mathcal{Q}$]
\label{assumption:real-q}
$\mathcal{Q}$ contains the optimal value function: $Q^\star\in\mathcal{Q}$.
\end{assumption}
On the other hand, we gather all the bounding assumptions here.
\begin{assumption}[Boundness of $\mathcal{Q}$]
\label{assumption:boundness-v}
For any $q\in\mathcal{Q}$, we assume $q\in(\mathcal{S}\times\mathcal{A}\to[0, V_{\textup{max}}])$. 
\end{assumption}
\begin{assumption}[Boundness of $\mathcal{B}$]
\label{assumption:boundness-b}
For any $\beta\in\mathcal{B}$, we assume $\beta\in(\mathcal{S}\times\mathcal{A}\to[0, U_{\mathcal{B}}])$.
\end{assumption}
\begin{assumption}[Boundness of $\mathcal{W}$]
\label{assumption:boundness-w}
For any $w\in\mathcal{W}$, we assume $w\in(\mathcal{S}\times\mathcal{A}\to[0, U_{\mathcal{W}}])$. 
\end{assumption}
\begin{remark}[Validity]
    The invertibility of $I-\gamma P_{\pi^\star_e}$ is shown by \Cref{lemma:invertible-sa} in \Cref{subsection:P-sa}.
    While \Cref{assumption:real-w,assumption:boundness-w} actually subsumes \Cref{assumption:opt-pi-con}, we make it explicit for clarity of explanation.
    \Cref{assumption:real-pi} implicitly assumes that $\pi_c$ covers $\pi^\star_e$, this can easily be done by directly choosing 
    $\pi_b$ as $\pi_c$. 
\end{remark}
\begin{remark}
    Although we include the normalization step in \Cref{assumption:real-pi}, this can also be achieved 
    with some preprocessing steps.
\end{remark}
\begin{remark}
    There is an overlap in the above assumptions: we can derive a policy ratio class $\mathcal{B}$ directly from $\mathcal{W}$ and $\mathcal{Q}$.
\end{remark}
With these prerequisites in place, we can finally state our finite-sample guarantee.
\begin{theorem}[Sample complexity of learning a near-optimal policy]
\label{theorem:finite1}
If \Cref{assumption:d-c,assumption:real-w,assumption:real-pi,assumption:real-q,assumption:opt-pi-con,assumption:boundness-w,assumption:boundness-v,assumption:boundness-b} hold with $\varepsilon_c\ge\frac{4C_cU_{\mathcal{B}}\sqrt{\varepsilon_{\textup{stat}}}}{1-\gamma}$ 
where
\begin{align*}
    \varepsilon_{\textup{stat}} = 
        U_{\mathcal{W}}V_{\textup{max}} \sqrt{\frac{2\log (2\lvert \mathcal{Q}\rvert \lvert \mathcal{W}\rvert/\delta)}{N_{\mathcal{D}}}}, 
\end{align*}
then with probability at least $1-\delta$, the output $\hat{\pi}$ from \Cref{algo} is near-optimal: 
\begin{align*}
    J^\star-J_{\hat{\pi}}\le\frac{4C_cU_{\mathcal{B}}\sqrt{\varepsilon_{\textup{stat}}}}{1-\gamma}.
\end{align*}
\end{theorem}
\paragraph{Proof sketch of \Cref{theorem:finite1}}
As we can obtain the near-optimality guarantee via \Cref{lemma:adv_to_subopt}, 
the remaining task is to approximate \Cref{eqn:0-adv}.
This comes from the following two lemmas.
\begin{lemma}[$L^2$ error of $\hat{q}$ under $d_c$, adapted from theorem 2 in \citet{Uehara2023RefinedVO}]
\label{lemma:l2-distance}
If \Cref{assumption:opt-pi-con,assumption:real-q,assumption:real-w,assumption:boundness-w,assumption:boundness-v} hold, 
with probability at least $1-\delta$, the estimated $\hat{q}$ from \Cref{algo} satisfies 
\begin{align*}
    \lVert \hat{q} - Q^\star\rVert_{d_c, 2}
    \le 2\sqrt{\varepsilon_{\textup{stat}}}.
\end{align*}
\end{lemma}
\begin{lemma}[From $L^1$ distance to \Cref{eqn:0-adv}]
\label{lemma:l1-2-adv}
If \Cref{assumption:real-pi,assumption:boundness-b} hold, 
    \begin{align*}
        \langle Q^\star(\cdot, \pi^\star_e) - Q^\star(\cdot, \hat{\pi}), \mu_c\rangle\le&
                2U_{\mathcal{B}}\lVert \hat{q}-Q^\star\rVert_{d_c, 1}.
    \end{align*}
\end{lemma}
Combine them, we have that with probability at least $1-\delta$, 
\begin{align*}
\langle Q^\star(\cdot, \pi^\star_e) - Q^\star(\cdot, \hat{\pi}), \mu_c\rangle\le&
2U_{\mathcal{B}}\lVert\hat{q}-Q^\star\rVert_{d_c, 1}\le2U_{\mathcal{B}} \lVert\hat{q}-Q^\star\rVert_{d_c, 2}\le4U_{\mathcal{B}}\sqrt{\varepsilon_{\textup{stat}}}. 
\end{align*}

\subsection{Comparison with related works}
\label{subsection:comparison}
We now provide a brief comparison of our method with some related algorithms.
\paragraph{Algorithms with gap assumptions} \citet{Chen2022OfflineRL} and \citet{Uehara2023RefinedVO} assume that there are
(soft) gaps in the optimal value function, which is only satisfied by part of MDPs, whereas our goal is to deal 
with general problems.
Moreover, 
while our algorithm is based on the optimal value estimator proposed by \citet{Uehara2023RefinedVO}, 
we use the policy ratio to ensure a finite distribution shift
and our near-optimality guarantee does not require the soft margin assumption. 
Besides, \citet{Uehara2023RefinedVO} use $d^\mathcal{D}$ as $d_c$, assuming that 
the dataset covers the optimal trajectories from itself.
This assumption is non-monotonic and hard to be satisfied in reality.
Instead, we propose using an additional covering distribution $d_c$ as an alternative, 
which can effectively utilize the prior knowledge about the optimal trajectories and trade off the dataset requirement.
\paragraph{Algorithms with behavior regularization}
\citet{Zhan2022OfflineRL} use behavior regularization to ensure that the learned policy is close to the dataset. 
Nevertheless, the regularization makes the optimality of the learned policy intractable.

\paragraph{Algorithms with pessimism in the face of uncertainty} These algorithms (e.g., 
\citet{Jiang2020MinimaxVI,Liu2020ProvablyGB,Xie2021BellmanconsistentPF,Cheng2022AdversariallyTA,Zhu2023ImportanceWA}%
) are often closely related to approximate dynamic programming (ADP).
They ``pessimistically'' estimate the given policies and update (or choose) policies ``pessimistically'' with the estimated value functions.
However, the evaluation step used in these algorithms always requires the strong realization of all candidate policies' value functions, which our algorithm avoids.

\paragraph{Limitations of our algorithm} 
On the one hand, the additional covering distribution may be hard to access in some scenarios, leading back to using $d^\mathcal{D}$ as $d_c$.
On the other hand, although mitigated with increasing dataset size, the assumption of covering all near-optimal policies is still stronger than the classic single-optimal concentrability.
In addition, the ``non-stationary'' coverage requirement is also somewhat restrictive.

\section{Conclusion and further work}
\label{section:discussion}
This paper present VOPR, a new MIS-based algorithm for offline RL with function approximations.
VOPR is inspired by the optimal value estimator proposed in \citet{Uehara2023RefinedVO}, and it circumvents the soft margin assumption in the original paper with the near-optimal coverage assumption.
While it still works if using the data distribution as the covering distribution, VOPR can trade off data assumptions with more refined choices.
Compared with other algorithms considering partial coverage, VOPR does not make strong function class assumptions and works under general MDPs.
Finally, despite the successes, a refined additional covering distribution may be difficult to obtain, and the near-optimal coverage assumption is still stronger than single optimal concentrability.
We leave them for further investigation.

\bibliography{ref,lp}

\clearpage

\appendix

\section{Notations}
\label{section:notations-table}
\begin{table}[H]
  \caption{Notations}
  \centering
  \label{notations-table}
  \begin{tabular}{lll}
    \toprule
    $\mathcal{S}$  & state space \\
    $\mathcal{A}$  & action space \\
    $\mathcal{Q}$     & state-action value function class  \\
    $\mathcal{W}$     & state-action distribution ratio function class  \\
    $\mathcal{B}$     & policy ratio function class  \\
    $\beta$     & members of $\mathcal{B}$  \\
    $V_\pi$     & state value function for policy $\pi$ \\
    $Q_\pi$     & state-action value function for policy $\pi$\\
    $V^\star$     & optimal state value function  \\
    $Q^\star$     & optimal state-action value function  \\
    $\nu$ & uniform measure of $\mathcal{A}$, $\mathcal{S}$, or $\mathcal{S}\times\mathcal{A}$, depending on the context \\
    $\mathcal{D}$   & dataset used in the algorithm \\
    $d^\mathcal{D}$   & state-action distribution of dataset \\
    $\mu^\mathcal{D}$   & state distribution of dataset \\
    $\pi_b$   & behaviour policy \\
    $d_c$  & the additional covering distribution\\
    $\mu_c$  & state distribution of the additional covering distribution\\
    $\pi_c$  & policy of the additional covering distribution\\
    $\langle a, b \rangle$ & inner product of $a$ and $b$, usually as $\int ab\ d\nu$\\
    $f_1\circ f_2$    & $(f_1\circ f_2)(s, a) = f_1(s,a)f_2(s,a)$, normalizing it if needed (e.g., density)\\ %
    $\mu\times \pi$    & $(\mu\times\pi)(s,a)=\mu(s)\pi(a\mid s)$\\ %
    $T^\star$ &  Bellman optimality operator, $T^\star q(s, a) \coloneqq R(s, a) + \gamma \mathbb{E}_{s^\prime\sim P(\cdot\mid s, a)}[\max q(s^\prime, \cdot)]$\\
    $\mu_0$  & initial state distribution\\
    $\mu_\pi^{i:j}$  & the $i$-th to $j$-th steps part of $\mu_\pi$\\
    $d_1\gg d_2$ & $d_2$ is absolutely continuous w.r.t. $d_1$\\
    $d_{\pi,i}$ & normalize $i$-th step part of state-action distribution induced by $\pi$\\
    $d_{d, \pi}$ & state-action distribution induced by $\pi$ from $d$\\
    $\pi_i$    & policy take $\pi$ in the previous $0$-th to $i$-th (include the $i$-th) steps, and take $\pi^\star_e$ after this\\
    $\pi_{\beta}$ & $\pi_{\beta}(a\mid s) = \pi_c(a\mid s)\beta(s,a)/\int_{\mathcal{A}}\pi_c(a\mid s)\beta(s,a)d\nu(a)$\\
    $\Pi_{\varepsilon,\textup{non}}^\star$ & the class of all non-stationary $\varepsilon$ near-optimal policies\\
    $P_\pi$ & state-action transition kernel with policy $\pi$\\
    $O^\star$ & conjucate operator of some operator $O$\\
    \bottomrule
  \end{tabular}
\end{table}

While $\pi, \mu$, and $d$ are mainly used to denote the Radon–Nikodym derivatives of the underlying probability measures w.r.t. $\nu$, 
we sometimes also use them to represent the correspondent distribution measure with abuse of notation.

\section{Helper Lemmas}

\subsection{Properties of $P_\pi$}
\label{subsection:P-sa}

We first provide some properties of $P_\pi$ (for any policy $\pi$) as an operator on the $L^\infty$-space of $\mathcal{S}\times\mathcal{A}$, 
and similar results should also hold for transition operators with policies defined on $\mathcal{S}$.
Note that the integrations of the absolute value of the functions considered in this subsection are always finite,
which means that we can change the orders of integrations via Fubini's Theorem.
As we will consider conjugate operators, we define the inner product as $\langle q, d\rangle = \int_{\mathcal{S}\times\mathcal{A}} q(s,a)d(s,a)d\nu(s, a)$.

\begin{lemma}
    $P_\pi$ is \emph{linear}.%
\end{lemma}
\begin{proof}
    Recall the definition of $P_\pi$, 
    \begin{align*}
        P_\pi d(s^\prime, a^\prime)=\int_{\mathcal{S}\times\mathcal{A}}\pi(a^\prime\mid s^\prime)P(s^\prime\mid s, a)d(s, a)d\nu(s, a)
    \end{align*}
    For any $d_1, d_2\in L^\infty(\mathcal{S}\times\mathcal{A})$, 
    \begin{align*}
        P_\pi(\alpha_1 d_1 + \alpha_2d_2)(s^\prime, a^\prime)=&\int_{\mathcal{S}\times\mathcal{A}}\pi(a^\prime\mid s^\prime)P(s^\prime\mid s, a)(\alpha_1 d_1 + \alpha_2d_2)(s, a)d\nu(s, a)\\=&
        \int_{\mathcal{S}\times\mathcal{A}}\alpha_1\pi(a^\prime\mid s^\prime)P(s^\prime\mid s, a)d_1(s, a)d\nu(s, a) +
        \int_{\mathcal{S}\times\mathcal{A}}\alpha_2\pi(a^\prime\mid s^\prime)P(s^\prime\mid s, a)d_2(s, a)d\nu(s, a)\\=&
        \alpha_1\int_{\mathcal{S}\times\mathcal{A}}\pi(a^\prime\mid s^\prime)P(s^\prime\mid s, a)d_1(s, a)d\nu(s, a) +
        \alpha_2\int_{\mathcal{S}\times\mathcal{A}}\pi(a^\prime\mid s^\prime)P(s^\prime\mid s, a)d_2(s, a)d\nu(s, a)\\=&
        \alpha_1P_\pi d_1(s^\prime, a^\prime)+ \alpha_2P_\pi d_2(s^\prime, a^\prime)
    \end{align*}
    This compeletes the proof.
\end{proof}

\begin{lemma}
    The adjoint operator of $P_\pi$ is
    \begin{align*}
        P_\pi^\star q(s,a)= \int_{\mathcal{S}\times\mathcal{A}}q(s^\prime, a^\prime)\pi(a^\prime\mid s^\prime)P(s^\prime\mid s, a)d \nu(s^\prime, a^\prime).
    \end{align*}
\end{lemma}
\begin{remark}
    Intuitively, we can see $P_\pi d(s^\prime, a^\prime)$ as one-step forward of $d$,
    such that we start from $(s,a)\sim d$, transit into $s^\prime\sim P(\cdot\mid s, a)$ and take $a^\prime\sim \pi(\cdot\mid s^\prime)$. 
    Also, we can view $P_\pi^\star q(s,a)$ as one-step backward of $q$, such that we compute the value of $(s,a)$ 
    through the one step transferred state-action distribution with the help of $q$.
\end{remark}
\begin{proof}
    Consider the inner products $\langle q, P_\pi d\rangle$ and $\langle P_\pi^\star q, d\rangle$, we should prove that
    these two are equal.
    By definition, 
    \begin{align*}
        \langle q, P_\pi d\rangle =&\int_{\mathcal{S}\times\mathcal{A}} \Bigg[q(s^\prime, a^\prime)\int_{\mathcal{S}\times\mathcal{A}}\pi(a^\prime\mid s^\prime)P(s^\prime\mid s, a)d(s, a)d\nu(s, a)\Bigg] d \nu(s^\prime, a^\prime)\\=&
        \int_{\mathcal{S}\times\mathcal{A}} \int_{\mathcal{S}\times\mathcal{A}}q(s^\prime, a^\prime)\pi(a^\prime\mid s^\prime)P(s^\prime\mid s, a)d(s, a)d\nu(s, a)d \nu(s^\prime, a^\prime)
    \end{align*}
    and 
    \begin{align*}
        \langle P_\pi^\star q, d\rangle =&
        \int_{\mathcal{S}\times\mathcal{A}} d(s, a)\Big[\int_{\mathcal{S}\times\mathcal{A}}q(s^\prime, a^\prime)\pi(a^\prime\mid s^\prime)P(s^\prime\mid s, a)d \nu(s^\prime, a^\prime)\Big]d\nu(s, a)\\=&
        \int_{\mathcal{S}\times\mathcal{A}} \int_{\mathcal{S}\times\mathcal{A}}d(s, a)q(s^\prime, a^\prime)\pi(a^\prime\mid s^\prime)P(s^\prime\mid s, a)d \nu(s^\prime, a^\prime)d\nu(s, a)\\=&
        \tag{Fubini's Theorem} \int_{\mathcal{S}\times\mathcal{A}} \int_{\mathcal{S}\times\mathcal{A}}q(s^\prime, a^\prime)\pi(a^\prime\mid s^\prime)P(s^\prime\mid s, a)d(s, a)d\nu(s, a)d \nu(s^\prime, a^\prime).
    \end{align*}
    This completes the proof.
\end{proof}

\begin{lemma-a}
    $\lVert P_\pi^\star\rVert_\infty = \lVert P_\pi\rVert_\infty \le 1$
\end{lemma-a}
\begin{remark}
    This upper bound should be intuitive since that $P_\pi$ can be seen as a probability transition kernel from $\mathcal{S}\times\mathcal{A}$ to itself.
\end{remark}
\begin{proof}
    Fix any $s\in\mathcal{S}$, $a\in\mathcal{A}$,
    we define $p(s^\prime, a^\prime) = P(s^\prime\mid s, a)\pi(a^\prime\mid s^\prime)$, 
    By Fubini's theorem, we have that %
    \begin{align*}
        \lVert p \rVert_{1, \nu} =&\int_{\mathcal{S}\times\mathcal{A}} \lvert p\rvert d\nu=
        \int_{\mathcal{S}\times\mathcal{A}} p d\nu\\=&
        \int_{\mathcal{S}}\int_{\mathcal{A}}P(s^\prime\mid s, a)\pi(a^\prime\mid s^\prime) d\nu(a^\prime)d\nu(s^\prime)\\=&
        \int_{\mathcal{S}}P(s^\prime\mid s, a)\Bigg[\int_{\mathcal{A}}\pi(a^\prime\mid s^\prime)d\nu(a^\prime)\Bigg]d\nu(s^\prime)\\=&
        \int_{\mathcal{S}}P(s^\prime\mid s, a)d\nu(s^\prime)\\=&
        1.
    \end{align*}
    For another function $q$ on $\mathcal{S}\times\mathcal{A}$ such that $\lVert q \rVert_{\infty,\nu}\le 1$, we can use Hölder's inequality, which yields
    \begin{align*}
        \lVert pq\rVert_{1, \nu}\le \lVert q\rVert_{\infty,\nu} \lVert p\rVert_{1, \nu} \le 1.
    \end{align*}
    Thus, for any $s\in\mathcal{S}, a\in\mathcal{A}$, and function $q$ with $\lVert q \rVert_{\infty,\nu}\le 1$, 
    \begin{align*}
        P_\pi^\star q(s, a)= 
        \int_{\mathcal{S}\times\mathcal{A}}q(s^\prime, a^\prime)\pi(a^\prime\mid s^\prime)P(s^\prime\mid s, a)d \nu(s^\prime, a^\prime)=
        \lVert pq\rVert_{1, \nu}\le 1.
    \end{align*}
    So, we have that
    \begin{align*}
        \lVert P_\pi\rVert_\infty =
        \lVert P_\pi^\star\rVert_\infty =
        \max\limits_{\lVert q\rVert_\infty\le 1} \lVert P_\pi^\star q\rVert_{\infty,\nu} \le \max\limits_{\lVert q\rVert_\infty\le 1}\max\limits_{s\in\mathcal{S}, a\in\mathcal{A}}P_\pi^\star q(s, a)\le 1.
    \end{align*}
    This completes the proof.
\end{proof}

\begin{lemma}
\label{lemma:invertible-sa}
    $I-\gamma P_\pi$ is invertible and 
        \begin{align*}
            (I-\gamma P_\pi)^{-1} = \sum\limits_{i=0}^{\infty}(\gamma P_\pi)^i. 
        \end{align*}
\end{lemma}
\begin{proof}
Since $\lVert P_\pi\rVert_\infty \le 1$, $\sum\limits_{i=0}^{\infty}(\gamma P_\pi)^i$ converges.
Take multiplication
    \begin{align*}
        (I-\gamma P_\pi)[\sum\limits_{i=0}^{\infty}(\gamma P_\pi)^i] = &
            \sum\limits_{i=0}^{\infty}(\gamma P_\pi)^i-
            \sum\limits_{i=1}^{\infty}(\gamma P_\pi)^i \\=&
        (\gamma P_\pi)^0 \\=&I.
    \end{align*}
This completes the proof.
\end{proof}
\begin{proposition}
    By definition, $d_{d, \pi}=(1-\gamma)\sum\limits_{i=0}^{\infty}(\gamma P_\pi)^id =(1-\gamma)(I-\gamma P_{\pi})^{-1}d$.
\end{proposition}

\subsection{Other useful lemmas}

\begin{lemma-a}[Performance difference lemma] We can decompose the performance gap as 
\label{lemma-a:pdl}
    \begin{align*}
        (1-\gamma)(J_{\pi_1} - J_{\pi_2}) = \langle \mu_{\pi_1}, Q_{\pi_2}(\cdot, \pi_1) - Q_{\pi_2}(\cdot, \pi_2)\rangle.
    \end{align*}
\end{lemma-a}
\begin{proof}
    By definition, 
    \begin{align*}
        \langle \mu_{\pi_1}, Q_{\pi_2}(\cdot, \pi_1) - Q_{\pi_2}(\cdot, \pi_2)\rangle=&
        \mathbb{E}_{s\sim\mu_{\pi_1}}\big[ R(\cdot, \pi_1) + \gamma \mathbb{E}_{a\sim\pi_1(\cdot\mid s), s^\prime\sim P(\cdot\mid s,a)}[Q_{\pi_2}(s^\prime, \pi_2)] - Q_{\pi_2}(\cdot, \pi_2)\rangle\big]\\=&
        \mathbb{E}_{s\sim\mu_{\pi_1}}\big[ R(\cdot, \pi_1)\big] + \mathbb{E}_{s\sim\mu_{\pi_1}}\big[\gamma \mathbb{E}_{a\sim\pi_1(\cdot\mid s), s^\prime\sim P(\cdot\mid s,a)}[Q_{\pi_2}(s^\prime, \pi_2)]\big] \\&- \mathbb{E}_{s\sim\mu_{\pi_1}}\big[Q_{\pi_2}(\cdot, \pi_2)\rangle\big]\\=&
        \mathbb{E}_{s\sim\mu_{\pi_1}}\big[ R(\cdot, \pi_1)\big] + \gamma\mathbb{E}_{s\sim\mu_{\pi_1},a\sim\pi_1(\cdot\mid s), s^\prime\sim P(\cdot\mid s,a)}[Q_{\pi_2}(s^\prime, \pi_2)]\big] \\&- \mathbb{E}_{s\sim\mu_{\pi_1}}\big[Q_{\pi_2}(\cdot, \pi_2)\rangle\big]\\=&
        \mathbb{E}_{s\sim\mu_{\pi_1}}\big[ R(\cdot, \pi_1)\big] + \mathbb{E}_{s\sim\mu_{\pi_1}}[Q_{\pi_2}(s, \pi_2)]\big]-
        (1-\gamma)\mathbb{E}_{s\sim\mu_0}[Q_{\pi_2}(s, \pi_2)]\big] \\&- \mathbb{E}_{s\sim\mu_{\pi_1}}\big[Q_{\pi_2}(\cdot, \pi_2)\rangle\big]\\=&
        \mathbb{E}_{s\sim\mu_{\pi_1}}\big[ R(\cdot, \pi_1)\big]- (1-\gamma)\mathbb{E}_{s\sim\mu_0}[Q_{\pi_2}(s, \pi_2)]\big] \\=&
        (1-\gamma) (J_{\pi_1} - J_{\pi_2})
    \end{align*}
    The first equality comes from Bellman equation, and the fourth equality comes from the definition of $\mu_\pi$.
    This completes the proof.
\end{proof}

\section{Detailed proofs for \Cref{section:adv2gap}}
\label{section:proof_adv2gap}

\subsection{Proof of \Cref{lemma:adv_to_subopt}}
\label{proof-adv_to_subopt}

\begin{lemma-a}[From advantage to optimality, restatement of \Cref{lemma:adv_to_subopt}]
If $\langle \mu_c, Q^\star(\cdot, \hat{\pi}) - Q^\star(\cdot, \pi^\star)\rangle\ge-\varepsilon$ , 
and \Cref{assumption:d-c} holds with $\varepsilon_c\ge\frac{C_c\varepsilon}{1-\gamma}$,   
$\hat{\pi}$ is $\frac{C_c\varepsilon}{1-\gamma}$ near-optimal.
\end{lemma-a}
\begin{proof}
    We begin with using induction to prove that $\hat{\pi}_i$ is $\frac{C_c\varepsilon}{1-\gamma}$ near-optimal for any $i\in\mathbb{N}$:
    \begin{itemize}
        \item We first show that $\hat{\pi}_0$ is $\frac{C_c\varepsilon}{1-\gamma}$ near-optimal.
        From \Cref{assumption:d-c}, we can use any $\tilde{\pi}\in\Pi^\star_{\varepsilon_c,\textup{non}}$ to conclude that
        \begin{align*}
            \Big\lVert\frac{\mu_0}
            {\mu_c}\Big\rVert_{\infty}\le
            \Big\lVert\frac{\mu_{\tilde{\pi}}/(1-\gamma)}
            {\mu_c}\Big\rVert_{\infty}\le \frac{C_c}{1-\gamma}.
        \end{align*}
        Thus, we can the show optimality of $\hat{\pi}^\star_0$ by the advantage: 
    \begin{align*}
        \langle \mu_{\hat{\pi}_0}, Q^\star(\cdot, \hat{\pi}_0) - Q^\star(\cdot, \pi^\star_e) \rangle =&
        \langle \mu^{0:0}_{\hat{\pi}_0}, Q^\star(\cdot, \hat{\pi}) - Q^\star(\cdot, \pi^\star_e) \rangle  +
        \langle \mu^{1:\infty}_{\hat{\pi}_0}, Q^\star(\cdot, \pi^\star_e) - Q^\star(\cdot, \pi^\star_e) \rangle \\=&
        \langle \mu^{0:0}_{\hat{\pi}^\star_0}, Q^\star(\cdot, \hat{\pi}) - Q^\star(\cdot, \pi^\star_e) \rangle\\=&
        (1-\gamma) \langle \mu_0, Q^\star(\cdot, \hat{\pi}) - Q^\star(\cdot, \pi^\star_e)\rangle\\\ge& %
        \tag{$Q^\star(\cdot, \hat{\pi}) - Q^\star(\cdot, \pi^\star_e)$ is non-positive} C_c \langle \mu_c, Q^\star(\cdot, \hat{\pi}) - Q^\star(\cdot, \pi^\star_e)\rangle\\\ge&
        -C_c \varepsilon.
    \end{align*}
    By performance difference lemma, 
    \begin{align*}
        (1-\gamma)(J_{\hat{\pi}_0}-J^\star)=&\langle\mu_{\hat{\pi}_0}, Q^\star(\cdot, \hat{\pi}_0) - Q^\star(\cdot, \pi^\star_e) \rangle
        \\\ge&-C_c\varepsilon.
    \end{align*}
    \item Next, we show that if $\hat{\pi}_i$ is $\frac{C_c\varepsilon}{1-\gamma}$ near-optimal, $\hat{\pi}_{i+1}$ is $\frac{C_c\varepsilon}{1-\gamma}$ near-optimal.
    Since that $\hat{\pi}_i$ is $\frac{C_c\varepsilon}{1-\gamma}$ optimal, the distribution shift from $\mu_c$ to $\mu_{\hat{\pi}_i}$ is bounded, which means, 
        \begin{align*}
            \Big\lVert\frac{\mu_{\hat{\pi}}^{0:i+1}}
            {\mu_c}\Big\rVert_{\infty}= 
            \Big\lVert\frac{\mu_{\hat{\pi}_i}^{0:i+1}}
            {\mu_c}\Big\rVert_{\infty}\le 
            \Big\lVert\frac{\mu_{\hat{\pi}_i}}
            {\mu_c}\Big\rVert_{\infty}\le C_c.
        \end{align*}
    Then, we have 
    \begin{align*}
        &\langle \mu_{\hat{\pi}_{i+1}}, Q^\star(\cdot, \hat{\pi}_{i+1}) - Q^\star(\cdot, \pi^\star_e) \rangle\\ =&
        \langle \mu^{0:i+1}_{\hat{\pi}_{i+1}}, Q^\star(\cdot, \hat{\pi}) - Q^\star(\cdot, \pi^\star_e) \rangle  +
        \langle \mu^{i+2:\infty}_{\hat{\pi}_{i+1}}, Q^\star(\cdot, \pi^\star_e) - Q^\star(\cdot, \pi^\star_e) \rangle \\=&
        \langle \mu^{0:i+1}_{\hat{\pi}_{i+1}}, Q^\star(\cdot, \hat{\pi}) - Q^\star(\cdot, \pi^\star_e) \rangle\\=&
        \langle \mu^{0:i+1}_{\hat{\pi}}, Q^\star(\cdot, \hat{\pi}) - Q^\star(\cdot, \pi^\star_e) \rangle\\\ge&
        \tag{$Q^\star(\cdot, \hat{\pi}) - Q^\star(\cdot, \pi^\star_e)$ is non-positive} C_c\langle \mu_c, Q^\star(\cdot, \hat{\pi}) - Q^\star(\cdot, \pi^\star_e)\rangle\\\ge& %
        -C_c\varepsilon.
    \end{align*}
    By performance difference lemma, 
    \begin{align*}
        (1-\gamma)(J_{\hat{\pi}_{i+1}}-J^\star)=&\langle\mu_{\hat{\pi}_{i+1}}, Q^\star(\cdot, \hat{\pi}_{i+1}) - Q^\star(\cdot, \pi^\star_e) \rangle.
        \\\ge&-C_c\varepsilon
    \end{align*}
    Therefore, $\hat{\pi}_{i+1}$ is $\frac{C_c\varepsilon}{1-\gamma}$ near-optimal.
    \end{itemize}
    Thus, for any $\epsilon>0$, there exists natural number $i\ge\log_{\gamma} \frac{\epsilon}{V_{\max}}$ such that 
    \begin{align*}
        J^\star-J_{\hat{\pi}}\le J^\star- J_{\hat{\pi}}^{0:i}\le J^\star - (J_{\hat{\pi}_i}-\gamma^{i+1} V_{\textup{max}})\le 
        \frac{C_c\varepsilon}{1-\gamma}+\gamma^{i+1}V_{\textup{max}}\le \frac{C_c\varepsilon}{1-\gamma}+\epsilon, 
    \end{align*}
    where $J_{\pi}^{i:j}$ denotes the $i$-th to $j$-th steps part of the return.
    Therefore, $\hat{\pi}$ is $\frac{C_c\varepsilon}{1-\gamma}$ near-optimal.
\end{proof}

\section{Detailed proofs for \Cref{section:algorithm-and-analysis}}
\label{section:proof_algorithm}

\subsection{Proof of \Cref{lemma:exactly-d}}
\begin{lemma-a}[Restatement of \Cref{lemma:exactly-d}]
    If $d_c$ is a linear combination of the state-action distributions induced by $\varepsilon$ near-optimal non-stationary policies $\Pi_{\varepsilon, \textup{non}}^\star$ under a fixed probability measure $\lambda$:
    \begin{align}
        d_c = \int_{\Pi_{\varepsilon, \textup{non}}^\star} d_{\tilde{\pi}}d\lambda(\tilde{\pi}).\label{eqn:best-d}
    \end{align}
    And $d^\mathcal{D}$ covers all admissible distributions of $\Pi_{\varepsilon, \textup{non}}^\star$:
    \begin{align*}
        \forall\ \tilde{\pi}\in\Pi^\star_{\varepsilon, \textup{non}},\ i\in\mathbb{N}, \ 
        \Big\lVert\frac{d_{\tilde{\pi},i}}{d^\mathcal{D}}\Big\rVert_{\infty} \le C.
    \end{align*}
    The distribution shift from $d^\mathcal{D}$ is bounded as 
    \begin{align*}
        \Big\lVert \frac{d_{d_c, \pi^\star_e}}{d^\mathcal{D}}\Big\rVert_{\infty}\le C.
    \end{align*}
\end{lemma-a}
\begin{proof}
    Define the state-action distribution of policy $\pi$ from $s\in\mathcal{S}, a\in\mathcal{A}$ at step $i$ as
    \begin{align*}
        d_{s, a, \pi, i}(s^\prime, a^\prime) = P(s_i=s^\prime, a_i=a^\prime\mid& s_0=s, a_0=a, s_1\sim P(\cdot\mid s_0, a_0), a_1\sim \pi(\cdot\mid s_1)\dots\\& s_j\sim P(\cdot\mid s_{j-1}, a_{j-1}), a_j\sim \pi(\cdot\mid s_j)\dots).
    \end{align*}
    Also, define the global version of it as 
    \begin{align*}
        d_{s, a, \pi}(s^\prime, a^\prime) = (1-\gamma)\sum\limits_{i=0}^{\infty}d_{s, a, \pi, i}(s^\prime, a^\prime).
    \end{align*}
    We can rewrite $d_{d_c, \pi^\star_e}(s, a)$ as 
    \begin{align*}
        d_{d_c, \pi^\star_e}(s, a)=&\int_{\mathcal{S}\times\mathcal{A}} d_{s_1,a_1,\pi^\star_e}(s, a)d_c(s_1, a_1)d\nu(s_1, a_1)\\=&
        \int_{\mathcal{S}\times\mathcal{A}} d_{s_1,a_1,\pi^\star_e}(s, a)
        \Big[\int_{\Pi}d_{\tilde{\pi}}(s_1,a_1) d\lambda(\tilde{\pi})\Big]
        d\nu(s_1, a_1)\\=&
        \tag{Fubini's Theorem} \int_{\Pi}\Big[\int_{\mathcal{S}\times\mathcal{A}} d_{s_1,a_1,\pi^\star_e}(s, a)
        d_{\tilde{\pi}}(s_1,a_1) 
        d\nu(s_1, a_1)\Big]d\lambda(\tilde{\pi})\\=&
        \int_{\Pi}\Big[\int_{\mathcal{S}\times\mathcal{A}} 
        (1-\gamma)\sum\limits_{i=0}^{\infty}\big[
        \gamma^i d_{s_1,a_1,\pi^\star_e}(s, a)
        d_{\tilde{\pi},i}(s_1,a_1) \big]
        d\nu(s_1, a_1)\Big]d\lambda(\tilde{\pi})\\=&
        \int_{\Pi}\Big[ 
        (1-\gamma)\sum\limits_{i=0}^{\infty}\big[
        \gamma^i \int_{\mathcal{S}\times\mathcal{A}}d_{s_1,a_1,\pi^\star_e}(s, a)
        d_{\tilde{\pi},i}(s_1,a_1)
        d\nu(s_1, a_1) \big]\Big]d\lambda(\tilde{\pi})\\=&
        \int_{\Pi}\Big[ 
        (1-\gamma)\sum\limits_{i=0}^{\infty}
            d_{\tilde{\pi}_i}^{i:\infty}(s, a)
         \Big]d\lambda(\tilde{\pi}).
    \end{align*}
    The last equation comes from that
    \begin{align*}
        &\gamma^i \int_{\mathcal{S}\times\mathcal{A}}d_{s_1,a_1,\pi^\star_e}(s, a) d_{\tilde{\pi},i}(s_1,a_1)d\nu(s_1,a_1)\\=&
        \gamma^i \int_{\mathcal{S}\times\mathcal{A}}d_{s_1,a_1,\pi^\star_e}(s, a)\Big[\int_{\mathcal{S}}\Big[\int_{\mathcal{A}}d_{s_2,a_2,\tilde{\pi},i}(s_1,a_1)\tilde{\pi}(a_2\mid s_2)d\nu(a_2)\Big]\mu_0(s_2)d\nu(s_2)\Big]d\nu(s_1,a_1)\\=&
        \tag{Fubini's Theorem} \int_{\mathcal{S}}\Big[\int_{\mathcal{A}}\Big[\gamma^i \int_{\mathcal{S}\times\mathcal{A}}d_{s_1,a_1,\pi^\star_e}(s, a)d_{s_2,a_2,\tilde{\pi},i}(s_1,a_1)d\nu(s_1,a_1)\Big]\tilde{\pi}(a_2\mid s_2)d\nu(a_2)\Big]\mu_0(s_2)d\nu(s_2), 
    \end{align*}
    since
    \begin{align*}
        &\gamma^i \int_{\mathcal{S}\times\mathcal{A}}d_{s_1,a_1,\pi^\star_e}(s, a)d_{s_2,a_2,\tilde{\pi},i}(s_1,a_1)d\nu(s_1,a_1)\\=&
        \gamma^i \int_{\mathcal{S}\times\mathcal{A}}(1-\gamma)\sum\limits_{k=0}^{\infty}\big[\gamma^kd_{s_1,a_1,\pi^\star_e,k}(s,a)\big]d_{s_2,a_2,\tilde{\pi},i}(s_1,a_1)d\nu(s_1,a_1)\\=&
        (1-\gamma)\sum\limits_{k=0}^{\infty}\Big[\gamma^{k+i}\int_{\mathcal{S}\times\mathcal{A}}d_{s_1,a_1,\pi^\star_e,k}(s,a)d_{s_2,a_2,\tilde{\pi},i}(s_1,a_1)d\nu(s_1,a_1)\Big]\\=&
        (1-\gamma)\sum\limits_{k=0}^{\infty}\big[\gamma^{k+i}d_{s_2,a_2,\tilde{\pi}_i, k+i}(s,a)\big]\\=&
        (1-\gamma)\sum\limits_{k=i}^{\infty}\big[\gamma^{k}d_{s_2,a_2,\tilde{\pi}_i, k}(s,a)\big]\\=&
        d_{s_2,a_2,\tilde{\pi}_i}^{i:\infty}(s,a), 
    \end{align*}
    we get
    \begin{align*}
        &\gamma^i \int_{\mathcal{S}\times\mathcal{A}}d_{s_1,a_1,\pi^\star_e}(s, a) d_{\tilde{\pi},i}(s_1,a_1)d\nu(s_1,a_1)\\=&
        \int_{\mathcal{S}}\Big[\int_{\mathcal{A}}\Big[d_{s_2,a_2,\tilde{\pi}_i}^{i:\infty}(s,a)\Big]\tilde{\pi}(a_2\mid s_2)d\nu(a_2)\Big]\mu_0(s_2)d\nu(s_2)\\=& 
        d_{\tilde{\pi}_i}^{i:\infty}(s, a).
    \end{align*}
    Finally, $\forall s\in\mathcal{S}, a\in\mathcal{A}$, 
    \begin{align*}
        \frac{d_{d_c, \pi^\star_e}(s, a)}{d^\mathcal{D}(s, a)}=&
        \int_{\Pi}\Big[ 
        (1-\gamma)\sum\limits_{i=0}^{\infty}
        \frac{d_{\tilde{\pi}_i}^{i:\infty}(s, a)}{d^\mathcal{D}(s,a)}
        \Big]d\lambda(\tilde{\pi})\\=&
        \int_{\Pi}\Big[ 
        (1-\gamma)\sum\limits_{i=0}^{\infty}
        \frac{(1-\gamma)\sum_{j=i}^{\infty}\gamma^jd_{\tilde{\pi}_i,j}(s, a)}{d^\mathcal{D}(s,a)}
        \Big]d\lambda(\tilde{\pi})\\=&
        \int_{\Pi}\Big[ 
        (1-\gamma)\sum\limits_{i=0}^{\infty}
        (1-\gamma)\sum_{j=i}^{\infty}\gamma^j\frac{d_{\tilde{\pi}_i,j}(s, a)}{d^\mathcal{D}(s,a)}
        \Big]d\lambda(\tilde{\pi})\\\le&
        \tag{$\tilde{\pi}\in\Pi_{\varepsilon, \textup{non}}^\star$ indicates $\tilde{\pi}_i\in\Pi_{\varepsilon, \textup{non}}^\star$}\int_{\Pi}\Big[ 
        C(1-\gamma)^2\sum\limits_{i=0}^{\infty}
        \sum\limits_{j=i}^{\infty}\gamma ^j\Big]d\lambda(\tilde{\pi})\\\le&
        \int_{\Pi}\Big[ 
        C(1-\gamma)^2\sum\limits_{i=0}^{\infty}
        \frac{\gamma^i}{1-\gamma}\Big]d\lambda(\tilde{\pi})\\\le&
        \int_{\Pi} 
        Cd\lambda(\tilde{\pi})\\=&
        C.
    \end{align*}
    This completes the proof.
\end{proof}

\subsection{Proof of \Cref{lemma:l2-distance}}

Note that the lemmas and proofs of this subsection are mainly adapted from \citet{Uehara2023RefinedVO},
similar statements could also be found in the original paper. 
However, since that we use $d_c$ to replace $d^\mathcal{D}$, we present them for clarity of explanation and to make our paper self-contained. 
We refer interested readers to the original paper for another detail.

We first define the expected version of \Cref{eqn:emp-L} as 
\begin{align*}
    \mathcal{L}(d, q, w) \coloneqq&0.5\mathbb{E}_{d}[q^2(s, a)]+ \mathbb{E}_{(s, a)\sim d^\mathcal{D}_w, r=R(s, a), s^\prime\sim P(\cdot\mid s, a)} \big[\gamma \max q(s^\prime, \cdot)+r - q(s, a)\big]\\=&
    0.5\mathbb{E}_{d}[q^2(s, a)]+ \mathbb{E}_{\mathcal{D}_w} \big[\gamma \max q(s^\prime, \cdot)+r - q(s, a)\big]
\end{align*}
where $d^\mathcal{D}_w = d^\mathcal{D}\circ w$, and $\mathbb{E}_{\mathcal{D}_w}$ denotes taking expectation with respect to the reweighted data collecting process.
\begin{lemma}[Expectation] The expected value of $\hat{\mathcal{L}}(d, q, w)$ w.r.t. the data collecting process is 
\label{lemma:expectation}
$\mathcal{L}(d, q, w)$:
    \begin{align*}
        \mathbb{E}_{\mathcal{D}}[\hat{\mathcal{L}}(d, q, w)]=
        \mathcal{L}(d, q, w).
    \end{align*}
\end{lemma}
\begin{proof}
Since only the second term of $\hat{\mathcal{L}}$ is random,
we additional define 
\begin{align*}
\hat{\mathcal{L}}_{\mathcal{W}}(q, w) \coloneqq
        \frac{1}{N_{\mathcal{D}}}\sum\limits_{(s, a, r, s^\prime)\in\mathcal{D}} \mathbb{E}_{\mathcal{D}}\Big[
        w(s, a)\big[\gamma \max q(s^\prime, \cdot)+r - q(s, a)\big].
\end{align*}
We can rearrange the expectation as follows, 
    \begin{align}
        \mathbb{E}_{\mathcal{D}}[\hat{\mathcal{L}}(d, q, w)]=&\mathbb{E}_{\mathcal{D}}\Big[0.5\mathbb{E}_{d}[q^2(s, a)]+ \hat{\mathcal{L}}_{\mathcal{W}}(q, w)\Big]\\=&
        \mathbb{E}_{\mathcal{D}}\Big[0.5\mathbb{E}_{d}[q^2(s, a)]\Big]+ \mathbb{E}_{\mathcal{D}}\Big[\hat{\mathcal{L}}_{\mathcal{W}}(q, w)\Big]\\=&
        0.5\mathbb{E}_{d}[q^2(s, a)]+ \mathbb{E}_{\mathcal{D}}\Big[\hat{\mathcal{L}}_{\mathcal{W}}(q, w)\Big]\label{eqn:p-exp}
    \end{align}
    Then, by the i.i.d. assumption of samples and linear property of MIS, 
    \begin{align*}
        \mathbb{E}_{\mathcal{D}}[\hat{\mathcal{L}}(d, q, w)]=&0.5\mathbb{E}_{d}[q^2(s, a)]+ 
        \mathbb{E}_{\mathcal{D}}\Bigg[\frac{1}{N_{\mathcal{D}}}\sum\limits_{(s, a, r, s^\prime)\in\mathcal{D}} \Big[
        w(s, a)\big[\gamma \max q(s^\prime, \cdot)+r - q(s, a)\big]\Big]\Bigg]\\=&
        0.5\mathbb{E}_{d}[q^2(s, a)]+ \frac{1}{N_{\mathcal{D}}}\sum\limits_{(s, a, r, s^\prime)\in\mathcal{D}} \mathbb{E}_{\mathcal{D}}\Big[
        w(s, a)\big[\gamma \max q(s^\prime, \cdot)+r - q(s, a)\big]\Big]\\=&
        0.5\mathbb{E}_{d}[q^2(s, a)]+ \mathbb{E}_{\mathcal{D}}\Big[
        w(s, a)\big[\gamma \max q(s^\prime, \cdot)+r - q(s, a)\big]\Big]\\=&
        0.5\mathbb{E}_{d}[q^2(s, a)]+ \mathbb{E}_{(s, a)\sim d^\mathcal{D}, r=R(s, a), s^\prime\sim P(\cdot\mid s, a)}\Big[
        w(s, a)\big[\gamma \max q(s^\prime, \cdot)+r - q(s, a)\big]\Big]\\=&
        0.5\mathbb{E}_{d}[q^2(s, a)]+ \mathbb{E}_{(s, a)\sim d^\mathcal{D}}\Big[
        w(s, a)\big[\mathbb{E}_{s^\prime\sim P(\cdot\mid s, a)}[\gamma \max q(s^\prime, \cdot)]+R(s, a) - q(s, a)\big]\Big]\\=&
        0.5\mathbb{E}_{d}[q^2(s, a)]+ \mathbb{E}_{(s, a)\sim d^\mathcal{D}_w}\Big[
        \mathbb{E}_{s^\prime\sim P(\cdot\mid s, a)}[\gamma \max q(s^\prime, \cdot)]+R(s, a) - q(s, a)\Big]\\=&
        0.5\mathbb{E}_{d}[q^2(s, a)]+ \mathbb{E}_{(s, a)\sim d_w^\mathcal{D}, r=R(s, a), s^\prime\sim P(\cdot\mid s, a)}
        \big[\gamma \max q(s^\prime, \cdot)+r - q(s, a)\big]\\=&
        \mathcal{L}(d, q, w).
    \end{align*}
    This compeletes the proof.
\end{proof}
\begin{lemma}[Concentration]
\label{lemma:conc}
    For any fixed $d$, with probability at least $1-\delta$, for any 
    $q\in\mathcal{Q}$, $w\in\mathcal{W}$, 
    \begin{align*}
        \Big\lvert\mathcal{L}(d, q, w) - \hat{\mathcal{L}}(d, q, w)\Big\rvert\le \varepsilon_{\textup{stat}}.
    \end{align*}
\end{lemma}
\begin{proof}
    The statistical error only comes from $\hat{\mathcal{L}}_{\mathcal{W}}$, as  
    \begin{align*}
        \Big\lvert\mathcal{L}(d, q, w) - \hat{\mathcal{L}}(d, q, w)\Big\rvert=&
        \tag{\Cref{lemma:expectation}} \Big\lvert\mathbb{E}_{\mathcal{D}}[\hat{\mathcal{L}}(d, q, w)] - \hat{\mathcal{L}}(d, q, w)\Big\rvert\\=&
        \tag{\Cref{eqn:p-exp}} \Big\lvert\mathbb{E}_{\mathcal{D}}[\hat{\mathcal{L}}_{\mathcal{W}}(q, w)] - \hat{\mathcal{L}}_{\mathcal{W}}(q, w)\Big\rvert.
    \end{align*}
    Since each entry of $\mathcal{L}_{\mathcal{W}}$ is bounded: 
    \begin{align*}
        \forall q\in\mathcal{Q}, w\in\mathcal{W}, a\in\mathcal{A}, s^\prime\in\mathcal{S}, \quad
        \Big\lvert w(s, a)\big[\gamma \max q(s^\prime, \cdot)+r - q(s, a)\big]\Big\rvert\le U_{\mathcal{W}}V_{\textup{max}}, 
    \end{align*}
    we can apply Hoeffding's inequality which yields that, for any $q\in\mathcal{Q}$, $w\in\mathcal{W}$, with probability at least $1-\delta/(\lvert \mathcal{Q}\rvert \lvert \mathcal{W}\rvert)$,
    \begin{align*}
        \Big\lvert\mathbb{E}_{\mathcal{D}}[\hat{\mathcal{L}}_{\mathcal{W}}(q, w)] - \hat{\mathcal{L}}_{\mathcal{W}}(q, w)\Big\rvert\le
        U_{\mathcal{W}}V_{\textup{max}} \sqrt{\frac{2\log (2\lvert \mathcal{Q}\rvert \lvert \mathcal{W}\rvert/\delta)}{N_{\mathcal{D}}}}.
    \end{align*}
    Finally, we can use union bound, rearranging terms to get that, for any fixed $d$, with probability at least $1-\delta$, for any 
    $q\in\mathcal{Q}$, $w\in\mathcal{W}$, 
    \begin{align*}
        \Big\lvert\mathcal{L}(d, q, w) - \hat{\mathcal{L}}(d, q, w)\Big\rvert\le
        U_{\mathcal{W}}V_{\textup{max}} \sqrt{\frac{2\log (2\lvert \mathcal{Q}\rvert \lvert \mathcal{W}\rvert/\delta)}{N_{\mathcal{D}}}}=\varepsilon_{\textup{stat}}
    \end{align*}
    This compeletes the proof.
\end{proof}

\begin{lemma}
\label{lemma:bound-of-diff-q}
If $w$ is non-negative $\nu$-a.e. (e.g., $w\in\mathcal{W}$), for any $q\colon \mathcal{S}\times\mathcal{A}\to [0, V_{\textup{max}}]$, 
\begin{align}
    \mathcal{L}(d, q, w) - \mathcal{L}(d, Q^\star, w) \ge
    0.5\langle d, q^2-(Q^\star)^2\rangle + 
    \langle (\gamma P_{\pi^\star_e}-I)d^\mathcal{D}_w, q-Q^\star\rangle.\label{eqn:L2Q}
\end{align}
\end{lemma}
\begin{proof}
    This result simply comes from the definition:
    \begin{align*}
        &\mathcal{L}(d, q, w) - \mathcal{L}(d, Q^\star, w) \\=&
        0.5\mathbb{E}_{d}[q^2(s)-(Q^\star)^2(s)]\\&\quad+
        \mathbb{E}_{\mathcal{D}_w} [\gamma \max q(s^\prime, \cdot) + r - q(s, a)]- 
        \mathbb{E}_{\mathcal{D}_w} [\gamma \max Q^\star(s^\prime, \cdot) + r - Q^\star(s, a)]\\=&
        0.5\mathbb{E}_{d}[q^2(s)-(Q^\star)^2(s)]\\&\quad+
        \mathbb{E}_{\mathcal{D}_w} [\gamma \max q(s^\prime, \cdot) + r - q(s, a)]- 
        \mathbb{E}_{\mathcal{D}_w} [\gamma Q^\star(s^\prime, \pi^\star_e) + r - Q^\star(s, a)]\\\ge&
        0.5\mathbb{E}_{d}[q^2(s)-(Q^\star)^2(s)]\\&\quad+
        \mathbb{E}_{\mathcal{D}_w} [\gamma q(s^\prime, \pi^\star_e) + r - q(s, a)]- 
        \mathbb{E}_{\mathcal{D}_w} [\gamma Q^\star(s^\prime, \pi^\star_e) + r - Q^\star(s, a)]\\=&
        0.5\mathbb{E}_{d}[q^2(s)-(Q^\star)^2(s)]\\&\quad+
        \mathbb{E}_{\mathcal{D}_w} [\gamma (q-Q^\star)(s^\prime, \pi^\star_e) - (q-Q^\star)(s, a)]\\=&
        \tag{Rewrite the expectation with inner products}
        0.5\langle d, q^2-(Q^\star)^2\rangle+
        \langle d^\mathcal{D}_w, (\gamma P^\star_{\pi^\star_e}-I)(q-Q^\star)\rangle\\=&
        \tag{conjugate} 0.5\langle d, q^2-(Q^\star)^2\rangle+
        \langle (\gamma P_{\pi^\star_e}-I)d^\mathcal{D}_w, q-Q^\star\rangle.
    \end{align*}
    This compeletes the proof.
\end{proof}
\begin{lemma}
\label{lemma:bound_of_L}
If \Cref{assumption:real-q} holds,
with probability at least $1-\delta$, for any $w\in\mathcal{W}$ and any state-action distribution $d$,
we have
    \begin{align}
        \mathcal{L}(d, \hat{q}, w) - \mathcal{L}(d, Q^\star, w) \le 2\varepsilon_{\textup{stat}}\label{eqn:L-stat}.
    \end{align}
\end{lemma}
\begin{proof}
    We can decompose \Cref{eqn:L-stat} as follows, 
    \begin{align*}
        \mathcal{L}(d, \hat{q}, w) - \mathcal{L}(d, Q^\star, w) = &
        \underbrace{\mathcal{L}(d, \hat{q}, w) - \hat{\mathcal{L}}(d, \hat{q}, w)}_{(1)}
        + \underbrace{\hat{\mathcal{L}}(d, \hat{q}, w) - \hat{\mathcal{L}}(d, \hat{q}, \hat{w})}_{(2)}\\&
        + \underbrace{\hat{\mathcal{L}}(d, \hat{q}, \hat{w}) - \hat{\mathcal{L}}(d, Q^\star, \hat{w}(Q^\star))}_{(3)}
        + \underbrace{\hat{\mathcal{L}}(d, Q^\star, \hat{w}(Q^\star))- \mathcal{L}(d, Q^\star, \hat{w}(Q^\star))}_{(4)} \\&
        + \underbrace{\mathcal{L}(d, Q^\star, \hat{w}(Q^\star)) - \mathcal{L}(d, Q^\star, w)}_{(5)}
    \end{align*}
    where $\hat{w}(q)=\argmax_{w\in\mathcal{W}}\hat{\mathcal{L}}(d, q, w)$.
    For the terms above, we have that: 
    \begin{itemize}
        \item $(2)$ and $(3)$ are non-positive since the optimization process.
        \item $(1)$ and $(4)$ could be bound by concentration.
        \item For $(5)$, as Bellman optimality equation holds, 
            \begin{align*}
                \forall s\in\mathcal{S}, a\in\mathcal{A}, \quad\mathbb{E}_{s^\prime\sim P(\cdot\mid s, a)} \big[
                    \gamma \max Q^\star(s^\prime, \cdot)\big] + R(s,a) - Q^\star(s, a) = 0.
            \end{align*}
            We have that
            \begin{align*}
                (5)=&\mathcal{L}(d, Q^\star, \hat{w}(Q^\star)) - \mathcal{L}(d, Q^\star, w)\\=&
                0.5\mathbb{E}_{d}[(Q^\star)^2(s, a)]+ \mathbb{E}_{\mathcal{D}_{\hat{w}(Q^\star)}} \big[\gamma \max q(s^\prime, \cdot)+r - q(s, a)\big]\\&-
                \Big[0.5\mathbb{E}_{d}[(Q^\star)^2(s, a)]+ \mathbb{E}_{\mathcal{D}_w}\big[\gamma \max Q^\star(s^\prime, \cdot)+r - Q^\star(s, a)\big]\Big]\\=&
                \mathbb{E}_{(s, a)\sim d_{\hat{w}(Q^\star)}^\mathcal{D}, r=R(s, a), s^\prime\sim P(\cdot\mid s, a)}\big[\gamma \max Q^\star(s^\prime, \cdot)+r - Q^\star(s, a)\big]\\&-
                \Big[\mathbb{E}_{(s, a)\sim d_w^\mathcal{D}, r=R(s, a), s^\prime\sim P(\cdot\mid s, a)}\big[\gamma \max Q^\star(s^\prime, \cdot)+r - Q^\star(s, a)\big]\Big]\\=&
                \mathbb{E}_{(s, a)\sim d_{\hat{w}(Q^\star)}^\mathcal{D}}\Big[\gamma \mathbb{E}_{s^\prime\sim P(\cdot, s, a)}[\max Q^\star(s^\prime, \cdot)]+R(s, a) - Q^\star(s, a)\Big]\\&-
                \mathbb{E}_{(s, a)\sim d_w^\mathcal{D}}\Big[\gamma \mathbb{E}_{s^\prime\sim P(\cdot, s, a)}[\max Q^\star(s^\prime, \cdot)]+R(s, a) - Q^\star(s, a)\Big]\\=&0.
            \end{align*}
    \end{itemize}
    Thus, we conclude that with probability at least $1-\delta$, 
    \begin{align*}
        \mathcal{L}(\hat{q}, w) - \mathcal{L}(Q^\star, w) \le &
        \underbrace{\mathcal{L}(\hat{q}, w) - \hat{\mathcal{L}}(\hat{q}, w)}_{(1)}
        + \underbrace{\hat{\mathcal{L}}(Q^\star, \hat{w}(Q^\star))- \mathcal{L}(Q^\star, \hat{w}(Q^\star))}_{(4)} \\\le&
        \lvert\mathcal{L}(\hat{q}, w) - \hat{\mathcal{L}}(\hat{q}, w)\rvert
        + \lvert\hat{\mathcal{L}}(Q^\star, \hat{w}(Q^\star))- \mathcal{L}(Q^\star, \hat{w}(Q^\star))\rvert \\\le&
        \tag{\Cref{lemma:conc}}2\varepsilon_{\textup{stat}}.
    \end{align*}
    This compeletes the proof.
\end{proof}

With lemmas above, it's time to prove \Cref{lemma:l2-distance}.
\begin{lemma-a}[$L^2$ error of $\hat{q}$ under $d_c$, restatement \Cref{lemma:l2-distance}]
If \Cref{assumption:opt-pi-con,assumption:real-q,assumption:real-w,assumption:boundness-w,assumption:boundness-v} hold, 
with probability at least $1-\delta$, the estimated $\hat{q}$ from \Cref{algo} satisfies 
\begin{align*}
    \lVert \hat{q} - Q^\star\rVert_{d_c, 2}
    \le 2\sqrt{\varepsilon_{\textup{stat}}}.
\end{align*}
\end{lemma-a}
\begin{proof}
    By \Cref{assumption:real-w}, $d^\mathcal{D}_{w^\star}=(I-\gamma P_{\pi^\star})^{-1}d_c Q^\star$, and from \Cref{lemma:bound-of-diff-q} we have
\begin{align*}
    \mathcal{L}(d_c, \hat{q}, w^\star) - \mathcal{L}(d_c, Q^\star, w^\star) \ge&
    0.5\langle d_c, \hat{q}^2-(Q^\star)^2\rangle - 
    \langle (I-\gamma P_{\pi^\star})(I-\gamma P_{\pi^\star})^{-1}d_c Q^\star, (\hat{q}-Q^\star)\rangle\\=&
    0.5\langle d_c, \hat{q}^2-(Q^\star)^2\rangle - 
    \langle d_cQ^\star, (\hat{q}-Q^\star)\rangle\\=&
    0.5\langle d_c, \hat{q}^2-(Q^\star)^2\rangle - 
    \langle d_c, Q^\star(\hat{q}-Q^\star)\rangle\\=&
    0.5\langle d_c, (\hat{q}-Q^\star)^2\rangle\\=&
    0.5\lVert\hat{q}-Q^\star\rVert_{d_c, 2}^2.
\end{align*}
Together with \Cref{lemma:bound_of_L}, with probability at least $1-\delta$, 
\begin{align*}
    0.5\lVert\hat{q}-Q^\star\rVert_{d_c, 2}^2\le
    \mathcal{L}(d_c, \hat{q}, w^\star) - \mathcal{L}(d_c, Q^\star, w^\star) \le 2\varepsilon_{\textup{stat}}.
\end{align*}
Rearrange this and we can get 
\begin{align*}
    \lVert\hat{q}-Q^\star\rVert_{d_c, 2} \le 2\sqrt{\varepsilon_{\textup{stat}}}
\end{align*}
    This compeletes the proof.
\end{proof}

\subsection{Proof of \Cref{lemma:l1-2-adv}}
\begin{lemma-a}[Restatement of \Cref{lemma:l1-2-adv}]
    If \Cref{assumption:real-pi,assumption:boundness-b} hold, 
    \begin{align*}
        \langle Q^\star(\cdot, \pi^\star_e) - Q^\star(\cdot, \hat{\pi}), \mu_c\rangle\le&
                2U_{\mathcal{B}}\lVert \hat{q}-Q^\star\rVert_{d_c, 1}.
    \end{align*}
\end{lemma-a}
\begin{proof}
    We can rearrange the above term as
    \begin{align*}
        \langle Q^\star(\cdot, \pi^\star_e) - Q^\star(\cdot, \hat{\pi}), \mu_c\rangle =&
        \langle Q^\star(\cdot, \pi^\star_e) - \hat{q}(\cdot, \pi^\star_e), \mu_c\rangle + 
        \langle \hat{q}(\cdot, \pi^\star_e)- \hat{q}(\cdot, \hat{\pi}), \mu_c\rangle \\&+ 
        \langle \hat{q}(\cdot, \hat{\pi})-Q^\star(\cdot, \hat{\pi}), \mu_c\rangle\\\le&
        \tag{\Cref{assumption:real-pi}} \langle Q^\star(\cdot, \pi^\star_e) - \hat{q}(\cdot, \pi^\star_e), \mu_c\rangle + 
        \langle \hat{q}(\cdot, \hat{\pi})-Q^\star(\cdot, \hat{\pi}), \mu_c\rangle\\\le&
        \lVert Q^\star(\cdot, \pi^\star_e) - \hat{q}(\cdot, \pi^\star_e)\rVert_{\mu_c, 1} + 
        \lVert \hat{q}(\cdot, \hat{\pi})-Q^\star(\cdot, \hat{\pi})\rVert_{\mu_c, 1}\\=&
        \lVert Q^\star - \hat{q}\rVert_{\mu_c\times\pi^\star_e, 1} + 
        \lVert \hat{q}-Q^\star\rVert_{\mu_c\times\hat{\pi}, 1}\\\le&
        2U_{\mathcal{B}}\lVert Q^\star - \hat{q}\rVert_{d_c, 1}
    \end{align*}
    The distribution shift comes from the fact that
    \begin{align*}
        \Big\lVert \frac{\mu\times\pi_1}{\mu\times\pi_2}\Big\rVert_{\infty} = 
        \Big\lVert \frac{\pi_1}{\pi_2}\Big\rVert_{\infty}, 
    \end{align*}
    and shifts from $\pi_c$ to $\pi^\star_e$ and  $\hat{\pi}$ are both bound by $U_{\mathcal{B}}$ due to \Cref{assumption:real-pi,assumption:boundness-b}.
    This completes the proof.
\end{proof}

\subsection{Proof of \Cref{theorem:finite1}}
\label{proof:finite1}
\begin{theorem-a}[Finite sample guarantee of \Cref{algo}, restatement of \Cref{theorem:finite1}]
If \Cref{assumption:d-c,assumption:real-w,assumption:real-pi,assumption:real-q,assumption:opt-pi-con,assumption:boundness-w,assumption:boundness-v,assumption:boundness-b} hold with $\varepsilon_c\ge\frac{4C_cU_{\mathcal{B}}\sqrt{\varepsilon_{\textup{stat}}}}{1-\gamma}$, 
then with probability at least $1-\delta$, the output $\hat{\pi}$ from \Cref{algo} is near-optimal: 
\begin{align*}
    J^\star-J_{\hat{\pi}}\le\frac{4C_cU_{\mathcal{B}}\sqrt{\varepsilon_{\textup{stat}}}}{1-\gamma}.
\end{align*}
\end{theorem-a}
\begin{proof}%
From \Cref{lemma:l2-distance}, we have that with probability at least $1-\delta$, 
    \begin{align*}
        \lVert\hat{q}-Q^\star\rVert_{d_c, 1} \le
        \lVert\hat{q}-Q^\star\rVert_{d_c, 2} \le 2\sqrt{\varepsilon_{\textup{stat}}}.
    \end{align*}
Then apply \Cref{lemma:l1-2-adv} to bound the weighted advantage, 
    \begin{align*}
        \langle Q^\star(\cdot, \pi^\star_e) - Q^\star(\cdot, \hat{\pi}), \mu_c\rangle\le&
        2U_{\mathcal{B}}\lVert \hat{q}-Q^\star\rVert_{d_c, 1}\le
        4U_{\mathcal{B}}\sqrt{\varepsilon_{\textup{stat}}}.
    \end{align*}
Finally, according to \Cref{lemma:adv_to_subopt}, $\hat{\pi}$ is $\frac{4C_cU_{\mathcal{B}}\sqrt{\varepsilon_{\textup{stat}}}}{1-\gamma}$ optimal.
This completes the proof.
\end{proof}

\end{document}